\documentclass{article}

\usepackage{iclr2027_conference,times}
\iclrfinalcopy

\usepackage[utf8]{inputenc} % allow utf-8 input
\usepackage[T1]{fontenc}    % use 8-bit T1 fonts
\usepackage{hyperref}       % hyperlinks
\usepackage{url}            % simple URL typesetting
\usepackage{booktabs}       % professional-quality tables
\usepackage{amsfonts}       % blackboard math symbols
\usepackage{nicefrac}       % compact symbols for 1/2, etc.
\usepackage{microtype}      % microtypography
\usepackage{xcolor}         % colors

\usepackage{amsmath}
\usepackage{amsthm}
\usepackage{thmtools}
\usepackage{amssymb}
\usepackage{graphicx}
\usepackage{tikz}
\usepackage{subcaption}
\usepackage{seqsplit}
\usepackage{enumitem}
\usepackage{lineno}

\definecolor{darkblue}{rgb}{0, 0, 0.5}
\hypersetup{colorlinks=true, citecolor=darkblue, linkcolor=darkblue, urlcolor=darkblue}

\newcommand{\BoS}{\text{\textup{\textsc{BoS}}}}
\newcommand{\Copy}{\text{\textup{\textsc{Copy}}}}

\newcommand{\EoI}{\text{\textup{\textsc{EoI}}}}
\newcommand{\EoS}{\text{\textup{\textsc{EoS}}}}

\newcommand{\fl}{\mathsf{fl}}

\newcommand{\calF}{\mathcal{F}}

\newtheorem{theorem}{Theorem}[section]
\newtheorem{lemma}[theorem]{Lemma}
\newtheorem{proposition}[theorem]{Proposition}
\newtheorem{corollary}[theorem]{Corollary}
\newtheorem{definition}[theorem]{Definition}
\newtheorem{remark}[theorem]{Remark}

\newtheorem{fact}[theorem]{Fact}

\DeclareMathOperator*{\flsum}{\textstyle\sum^{\!\mathsf{fl}}}

\usepackage[nameinlink,noabbrev]{cleveref}

\title{Copying Under Repetition:\\From Prefix Matching to Position Hashing}
\author{%
\normalfont
\begin{tabular}[t]{@{}p{0.46\textwidth}@{\hspace{0.04\textwidth}}p{0.46\textwidth}@{}}
\textbf{Anxin Guo} & \textbf{Tianlong Huang}\\
Northwestern University & Duke University\thanks{Work partially conducted while at Toyota Technological Institute at Chicago.}\\
\texttt{anxinbguo@gmail.com} & \texttt{tianlong.huang@duke.edu}\\[1.5ex]
\textbf{Chenxiao Yang} & \textbf{Zhiyuan Li}\\
Toyota Technological Institute at Chicago & Toyota Technological Institute at Chicago\\
\texttt{chenxiao@ttic.edu} & \texttt{zhiyuanli@ttic.edu}
\end{tabular}
}
  
\begin{document}

\maketitle
\lhead{Preprint}

\begin{abstract}
Frontier LLMs can copy random text almost flawlessly, yet they often overrun or truncate long blocks of repeated characters.
In repeated strings, matching local content no longer reliably identifies where copying should continue or when it should stop.
Nevertheless, we show that a two-layer transformer with $O(\log\log N)$-bit precision and $O(\log N/\log\log N)$ width can copy every string up to a chosen design horizon $N$ and terminate correctly under arbitrary repetition.
We use RoPE-style rotations of attention values and outputs (VO-rotations) with small prime periods to encode distances from delimiter anchors modulo those primes.
By the Chinese Remainder Theorem, these periodic encodings jointly distinguish all required distances, allowing one retrieval rule to select either the next source token or the end marker.
In contrast, we prove a width-independent impossibility result for copying under sequential finite-precision softmax accumulation.
It applies to transformers whose only positional input is an additive attention bias that eventually becomes constant, including NoPE and a recursive implementation of ALiBi.
Once an identical-token run exceeds the resulting saturation threshold, further repetitions leave the logits unchanged.
If this threshold grows sublinearly with $N$, these models cannot both complete a long all-identical copy and stop correctly, with error modes resembling those of frontier LLMs.
In synthetic copying experiments, baseline transformer models fail on long repeated strings even within the training range. Trained VO models, by contrast, generalize beyond the training lengths with near-perfect accuracy up to length $\approx5005$, the selected primes' period product.

\end{abstract}

% !TEX root = ../main.tex

\section{Introduction}
\label{sec:intro}

Autoregressively copying text is a fundamental primitive for modern large language models (LLMs).
LLMs copy when they reuse retrieved documents, quote code or data, or transfer a specific answer from a chain of thought into the final response.
Historically, this copying task has motivated explicit pointer or copy mechanisms in sequence models~\citep{vinyals2015pointer, gu2016incorporating, see2017get}.
Meanwhile, transformers seem to learn to copy easily, often through the ``induction head'' mechanism~\citep{olsson2022incontext}: they match the already-copied prefix to its earlier occurrence in the prompt and retrieve the subsequent token.
For random sequences, short local contexts uniquely identify the occurrence.
\citet{jelassi2024repeat} formalized this content-addressed mechanism, showing that a two-layer transformer (with a particular position encoding) can copy strings of length $\leq L$ with $O(\log \log L)$ precision—provided the input contains no repeated $n$-grams.

\begin{figure}[!t]
    \centering
    \includegraphics[width=\textwidth]{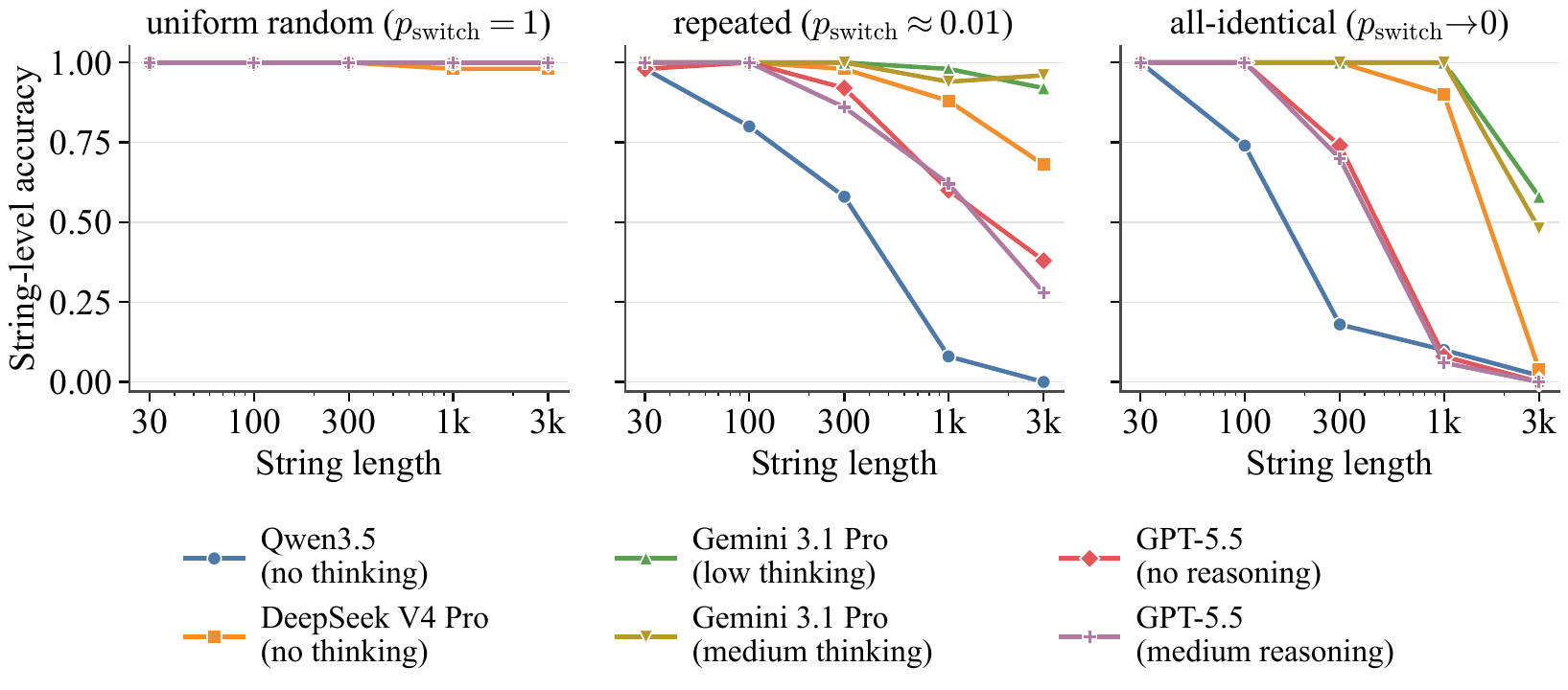}
    \caption{Long repeated strings reduce LLM copying accuracy.
    Accuracy is the fraction of all $50$ trials per length and condition whose extracted answer matches the source perfectly.
    For a fixed switch probability \(p_{\rm switch}\), a source string is generated by concatenating latent constant-letter runs: each run chooses a letter uniformly from the 26 lowercase English letters and has length \(R\sim\operatorname{Geom}(p_{\rm switch})\).
    The three plotted distributions are uniform-random strings, \(p_{\rm switch}=1\); repeated strings, \(p_{\rm switch}\approx0.01\); and all-identical strings, the limit \(p_{\rm switch}\to0\).
    % Adjacent latent runs may choose the same letter, so observed same-letter stretches can be longer than one run.
    Output extraction and experimental details are in Appendix~\ref{app:llm_experiment}.
    }
    \label{fig:llm_copying}
\end{figure}

Surprisingly, however, removing this non-repetition assumption seems to change the difficulty of copying qualitatively.
As shown in Figure~\ref{fig:llm_copying}, frontier LLMs that copy random strings near-perfectly sometimes fail catastrophically on strings with long repeated runs of letters, either overgenerating severely or stopping early.
Repetition changes what the task demands: when many positions carry the same local content, matching the already-copied prefix no longer picks out which occurrence to continue from, so the model must tell these occurrences apart by other means.

To make this problem concrete, consider one step of autoregressive copying, with current context
\[
  \langle\BoS\rangle,\; s_1,\ldots,s_L,\; \langle\EoI\rangle,\; \langle\Copy\rangle,\; s_1,\ldots,s_k,
\]
for some $k \leq L$.
The task is to output $s_{k+1}$ if $k<L$, or output $\langle\EoS\rangle$ if $k=L$.\footnote{$\langle\EoI\rangle$ stands for ``end of input'' and $\langle\EoS\rangle$ for ``end of sequence''.}
An all-identical source, with $s_i=\sigma$ for every~$i$, isolates the termination problem: continuing requires only another $\sigma$, but stopping at the right point requires distinguishing progress that local content windows cannot reveal.
This raises the question: Can a shallow, low-precision transformer copy reliably under arbitrary repetition and terminate correctly?

The answer is yes: for any fixed finite vocabulary and each sufficiently large design horizon~$N$, a single two-layer transformer can copy every source of length at most~$N$ under arbitrary repetition and terminate correctly, using $O(\log\log N)$-bit precision and $O(\log N/\log\log N)$ width.
Our construction uses its first layer to encode each position's \textit{offset relative to delimiters}—specifically, its distance to the preceding delimiter ($\langle\BoS\rangle$ in the source or $\langle\Copy\rangle$ in the copy segment)—so that tokens can be retrieved by position rather than by matching their content.
The need for a delimiter anchor arises because on an all-identical source, attention cannot extract positional signals from the source tokens themselves.
RoPE's query/key rotations~\citep{su2021roformer} and ALiBi's logit biases~\citep{press2022alibi} change attention weights but leave the source value vectors identical: redistributing a fixed total attention weight among these tokens does not change their contribution to the output.
Position can still affect the output through the weights assigned to special delimiters, such as the beginning-of-source token $\langle\BoS\rangle$, but only by changing how much of each delimiter's fixed value enters the output.
% \bob{relatively indirect and inefficient}
By contrast, value/output rotations (VO-RoPE; \citealp{su2025vorope}) let position act on the returned vector itself: even when two queries attend entirely to $\langle\BoS\rangle$, the returned vectors can differ because the rotation encodes each query's distance from it.

Our first layer applies these rotations with small prime periods (e.g., $2,3,\ldots$) to encode offsets relative to the nearest delimiter.
For each period $p$, the rotation encodes the offset modulo $p$, using a few well-separated phases that remain distinguishable at very low precision.
By the Chinese Remainder Theorem, the periods jointly distinguish every offset below their product, so the first layer resolves a long range of offsets with low precision and width.
The second layer then uses these representations to retrieve either the next source token or the end marker.

\subsection{Our Contributions}
We formalize the expressivity question in a \textit{horizon-uniform} framework: one model must copy every source of length $L\le N$ with the same parameters.
Its width, depth, and precision may grow with the design horizon $N$, but its parameters cannot specialize to the realized length.
This captures a natural requirement for a model designed for a finite context window, and lies between length-specific circuit expressivity~\citep{li2024chain,merrill2023parallelism} and fixed-resource expressivity over unbounded lengths~\citep{yang2025knee,yang2026length}.
Copying, for instance, is easy in the circuit model (the next-token map admits linear-size uniform $\mathrm{AC}^0$ circuits; see Appendix~\ref{app:ac0}), but fixed models also cannot copy inputs of unbounded lengths~\citep{huang2025formal,li2025fixedprecision}. 
By bridging these two regimes, our framework provides a fine-grained lens to quantify the positional representations, bit precision, and width needed for a transformer to copy reliably across all lengths within the horizon.

\paragraph{Copying and termination with VO-rotation.}
For any fixed finite vocabulary and all sufficiently large design horizons~$N$, we construct a two-layer transformer that copies every string of length $L\le N$ and terminates correctly under arbitrary repetition.
Our \textit{VO-rotation} instantiates the VO-RoPE operation with a prime-number period per head, using $\Theta(\log N/\log\log N)$ heads in total.
The construction uses $O(\log\log N)$-bit precision and $O(\log N/\log\log N)$ width (\Cref{thm:construction}).
We use small periodic counters to index the rounded rotation tables.
% Our \textit{VO-rotation} instantiates the VO-RoPE operation with a prime-number period per head.
% Layer~1 selects an anchor: the beginning-of-source delimiter $\langle\BoS\rangle$ in the source segment, or the copy-start delimiter $\langle\Copy\rangle$ in the copy segment.
% Layer~2 retrieves by the resulting offset representations, selecting either the next source token or the end-of-input marker, which the readout converts to the output termination marker (\Cref{sec:setup}).
% A head of period $q$ encodes the offset modulo~$q$; by the Chinese Remainder Theorem, $\Theta(\log N/\log\log N)$ prime-period heads distinguish offsets within the horizon.
% For any fixed finite vocabulary and all sufficiently large~$N$, the construction copies every string of length $L\le N$ with $O(\log\log N)$-bit precision and $O(\log N/\log\log N)$ width, using small periodic counters to index rounded rotation tables for VO rotation (\Cref{sec:construction}).

\paragraph{A width-independent copying obstruction.}
We also prove a width-independent impossibility result for \textit{eventually-constant bias-only} (ECB) transformers, whose only positional information enters through a pre-softmax additive attention bias that becomes constant beyond some distance.
Under our finite-precision execution model, the analysis covers NoPE~\citep{kazemnejad2024impact}, a recursive implementation of ALiBi~\citep{press2022alibi}, and bounded-window Hard-ALiBi~\citep{jelassi2024repeat} within this class.
The proof tracks identical-token runs after an arbitrary fixed prefix, layer by layer. 
Beyond a certain length, appending another identical token no longer changes the next-token logits.
In copying, a model that emits the required repeated letters then keeps repeating them for as long as the context remains in its domain, a failure pattern resembling the overruns observed in LLMs.
Increasing width does not remove this obstruction: for the instances in Corollary~\ref{cor:main}, models of depth $D(N)$ and precision $p(N)$ fail whenever $D(N)2^{p(N)}=o(N)$.
At the $O(\log\log N)$-bit precision of our construction, this rules out NoPE models of any width and depth $O(N^{1-\epsilon})$ for each fixed $\epsilon\in(0,1)$, whereas the VO construction uses just two layers.

\paragraph{Learning to copy under repetition.}
In synthetic experiments, all evaluated standard positional baselines lose accuracy under repetition near the training limit, while a shallower, narrower VO model learns to copy across all three evaluated distributions (\Cref{sec:experiments}).
VO copies perfectly at all evaluated lengths within the training range and generalizes with near-perfect accuracy up to the period product $5005$.
The models appear to have learned to utilize the periodic value-side positions, as accuracy on uniform-random and moderately repeated strings drops sharply just beyond $5005$ (\Cref{fig:synthetic_vo_generalization_10000}).

% !TEX root = ../main.tex
\section{Related Work}
\label{sec:related}

Copying under repetition requires distinguishing occurrences with the same content.
\citet{jelassi2024repeat} construct a two-layer transformer that copies exponentially long strings by content-based $n$-gram hashing, assuming no repeated $n$-grams.
Our two-layer construction instead retrieves by anchor-relative position, allowing arbitrary repetition and including correct termination in the guarantee. Repetition also distinguishes easy and difficult copying tasks in studies of length generalization~\citep{zhou2023algorithms}. \citet{huang2025formal} prove that certain Limit Transformers cannot copy arbitrary strings.
\citet{raju2026errors} model how attention errors accumulate with length in tasks such as list reversal, whereas we show how finite-precision saturation prevents correct termination in the ECB setting.

Under a positional-decay assumption, \citet{barbero2024glasses} show that, for a fixed model, the last-token representations of a growing sequence and of the same sequence with its final token repeated become arbitrarily close.
In the finite-precision ECB class, we instead prove exact terminal-run invariance above an explicit, width-independent threshold, even as depth and precision grow with the design horizon.
\citet{meyer2026outputs} bound the number of output sequences a fixed transformer can generate.
One of their bounds assumes invariance to limited token duplications or deletions, similar to the invariance we prove for long terminal runs.

Positional encodings differ both in where they enter the computation and in how they represent position.
Value-side representations include the additive relative embeddings of \citet{shaw2018selfattention} and the geometric transformations of GTA~\citep{miyato2024gta}. \citet{su2025vorope} gives the VO-RoPE mechanism used here, applying opposite position-dependent rotations to values before aggregation and to outputs afterward to encode relative offsets. RoVE~\citep{garciacastellanos2026rove} combines the same value/output rotation with query/key RoPE and studies it through an attentive-convolution interpretation and language-modeling experiments. For modular representations, ConveRT~\citep{henderson2020convert} uses learned absolute-position embeddings indexed by two coprime moduli, whereas we use \textit{prime-period value rotations} that encode anchor-relative offsets as modular residues (\Cref{thm:construction}).

Expressivity results depend on which resources may grow with input length.
Circuit complexity bounds characterize transformer computation~\citep{merrill2023parallelism,strobl2024formal}, but the fixed-context next-token copy map already lies in DLOGTIME-uniform $\mathrm{AC}^0$ (\Cref{app:ac0}). Temporal-logic analyses characterize fixed-precision NoPE transformers and the effect of local attention~\citep{li2025fixedprecision,li2026local}, while length-generalization theory studies when bounded-length training suffices for unbounded-length correctness~\citep{chen2025nonasymptotic,yang2026length}. Our horizon-uniform framework asks how positional and computational resources must scale with a finite design horizon, requiring the same parameters to work across all source lengths within it.

Finally, alternative routes to copying include sequential depth and explicit identifiers. \citet{fan2025looped} implement copying by iterating a fixed block for $n$ steps, and \citet{kraus2026barriers} use a growing signpost alphabet to address copying and retrieval barriers in chain-of-thought reasoning. 
% Our construction uses a fixed vocabulary and two layers, without intermediate reasoning or indexing tokens.

\section{Theoretical Framework}\label{sec:theory}

% Under repetition, content matching need not distinguish the occurrence to copy next.
We study two ways positional information can reach the output: through attention weights over value vectors, or through position-dependent transformations of the values themselves.
% The results establish failure for the ECB model class under a finite-precision execution model, and sufficiency of a modular value-side construction for two-layer copying.
Appendix~\ref{app:theory_proofs} gives the full execution model, architecture definitions, and proofs.

%% ==================================================================
\subsection{Horizon-Uniform Copying}\label{sec:setup}

Fix a design horizon $N\in\mathbb N$.
Let $\Gamma=\mathcal V\cup\{\langle\BoS\rangle,\langle\EoI\rangle,\langle\Copy\rangle,\langle\EoS\rangle\}$ be the token alphabet for the copy task.
For a transformer~$T$, write $g_T$ for its deterministic greedy next-token map, under a fixed tie-breaking rule, on contexts in its \emph{model domain}, i.e., the contexts that fit within its context window.
We consider transformers~$T_N$ with context window at least $2N+3$.
For any source length $L\le N$, the input prompt has the form
\[
  P(s):=
  \underbrace{\langle\BoS\rangle,\; s_1,\ldots,s_L,\; \langle\EoI\rangle}_{\text{source segment}},\; \langle\Copy\rangle ,
\]
where $s_1,\ldots,s_L\in\mathcal V$ and there is no restriction on repetition.
After seeing $\langle\Copy\rangle$, the model must autoregressively output $s_1,\ldots,s_L,\langle\EoS\rangle$.

\begin{definition}[Horizon-uniform copying]
\label{def:horizon-uniform-copying}
Fix a design horizon~$N$.
A model~$T_N$ solves $\mathrm{Copy}_{\le N}$ if the same parameters, execution convention, and readout satisfy the following condition for every source string $s\in\mathcal V^L$ with $L\le N$:
starting from the prompt~$P(s)$, it recursively produces
\[
  y_r
  =
  g_{T_N}\bigl(P(s),y_1,\ldots,y_{r-1}\bigr),
  \qquad
  1\le r\le L+1,
\]
with $(y_1,\ldots,y_L,y_{L+1}) = (s_1,\ldots,s_L,\langle\EoS\rangle)$.
A family of models $\{T_N\}$ solves horizon-uniform copying if each $T_N$ solves $\mathrm{Copy}_{\le N}$.
% Unless stated otherwise, we impose no computability or algorithmic uniformity requirement on the family $\{T_N\}_N$ across different horizons.
\end{definition}

% For the lower bound we apply a terminal-run fixed-point lemma to one all-identical source length.
% Once the terminal run is saturated, appending another copy of the same token no longer changes the next-token map.

\paragraph{Finite precision.}\label{sec:float}
Both our results assume deterministic finite-precision floating-point execution with finite arithmetic intermediates, as detailed in Appendix~\ref{app:fp_primitives}.
Throughout this section, $p=p(N)$ is a resource that may scale with the design horizon~$N$.
For the softmax numerator and denominator, we use the left-to-right chain of rounded binary additions from~\citet[Definition~3.1]{li2024chain}.
This execution model lets us track how repeated contributions lose length information through rounding, and how value-side offsets remain usable under the same numerical arithmetic.

The proofs rely on the following two facts about finite-precision arithmetic:
\begin{itemize}[leftmargin=*,nosep]
\item When intermediates remain finite, repeated addition and rounding of an identical summand reaches a fixed point after at most $2^p$ states: the sequence $a, \mathsf{fl}(a+b), \mathsf{fl}(\mathsf{fl}(a+b)+b), \ldots$ is eventually constant.
\item With the max-shifted softmax convention, sufficiently large post-shift logit gaps yield \textit{one-hot attention}, which is used in our VO construction.
% \item For the recursive implementation of ALiBi, the corollary assumes finite intermediates do not overflow; the full arithmetic, rounding, accumulation, underflow, and finite-intermediate conventions are specified in Appendix~\ref{app:fp_primitives}.
\end{itemize}

%% ==================================================================
\subsection{Lower Bound for Eventually-Constant Bias-Only Transformers}
\label{sec:impossibility}

\begin{definition}[Eventually-constant bias-only transformer class]\label{def:ecb}
A family $\{T_N\}_{N\in\mathbb N}$ is in $\mathcal T_{\mathrm{ECB}}$ if, for each horizon~$N$, $T_N$ is a depth-$D(N)$ causal transformer with an arbitrary number of attention heads and hidden dimensions, using deterministic $p(N)$-bit execution.
Its only positional dependence is a distance-dependent additive bias on the attention logits. For head $h$ in layer $\ell$, the attention logit from query position~$i$ to key position~$j$ is defined as:
\[
  e^{(\ell,h)}_{i,j}
  =
  \mathsf{fl}\bigl((q^{(\ell,h)}_i)^\top k^{(\ell,h)}_j - b^{(\ell)}_h(i-j)\bigr),
\]
where each $b^{(\ell)}_h$ is an eventually constant function: for some $n^{(\ell)}_h$, $b^{(\ell)}_h(d+1)=b^{(\ell)}_h(d)$ for all $d\ge n^{(\ell)}_h$.
Here ECB stands for \emph{eventually-constant bias-only}.
The full formal definition is in Definition~\ref{def:ecb_full}.
\end{definition}

Write \((x,\sigma^k)\) for a fixed prefix \(x\) followed by \(k\) copies of \(\sigma\).
The proof tracks where the hidden states become identical within each layer.
Let $\bar n_\ell=\max_h n_h^{(\ell)}$ and define $\kappa_0=0$, $\kappa_\ell=\kappa_{\ell-1}+\bar n_\ell+2^p$.

Once the previous layer's run states have saturated, a query at run index $j\ge\kappa_\ell$ has the same query vector for each~$j$. Each head scans three blocks (\Cref{fig:layerwise_saturation}): a fixed past beyond the bias threshold, a middle block where both the states and the bias are constant, and a near window at fixed relative distances.
The extra $\bar n_\ell$ positions account for the near window and leave at least $2^p$ identical contributions in the middle block. Increasing~$j$ only repeats those contributions, so neither the maximum visible logit nor the accumulator state upon leaving the block changes.
Repeating across all $D$ layers gives $\kappa_D$, beyond which extending the run leaves the final hidden state and next-token logits unchanged.

\begin{figure}[!t]
\centering
\begin{tikzpicture}[x=1cm,y=1cm,font=\footnotesize]
\definecolor{satblue}{RGB}{45,112,158}
\definecolor{satgreen}{RGB}{48,133,109}
\definecolor{satpurple}{RGB}{125,92,150}
\node[anchor=west,font=\small\bfseries] at (0,0.5) {(a) Each attention head: softmax sums from left to right};
\filldraw[fill=black!7,draw=black!45] (0,0) rectangle (3,-0.7);
\filldraw[fill=satblue!16,draw=satblue] (3,0) rectangle (8.2,-0.7);
\filldraw[fill=black!15,draw=black!55] (8.2,0) rectangle (12,-0.7);
\node[align=center] at (1.5,-0.35) {Earlier positions\\same fixed context};
\node at (5.6,-0.35) {Identical contributions};
\node[anchor=north] at (5.6,-0.8) {Same states in the previous layer};
\node[align=center] at (10.1,-0.35) {Nearby positions\\same distances every time};
\node[anchor=north east] at (12,-0.8) {Current token};
\draw[->,thick,satblue] (0,-1.4) -- (12,-1.4);
\node[anchor=north,font=\footnotesize] at (6,-1.45) {Summation order};
\node[anchor=west,font=\small\bfseries] at (0,-2.45) {(b) Each layer has a long tail of identical states};
\node[anchor=east] at (1.5,-3.05) {Input};
\node[anchor=east] at (1.5,-3.7) {Layer $1$};
\node[anchor=east] at (1.5,-4.35) {Layer $2$};
\filldraw[fill=satblue!16,draw=satblue] (2,-2.83) rectangle (12,-3.27);
\filldraw[fill=black!7,draw=black!35] (2,-3.48) rectangle (4.7,-3.92);
\filldraw[fill=satgreen!17,draw=satgreen] (4.7,-3.48) rectangle (12,-3.92);
\filldraw[fill=black!7,draw=black!35] (2,-4.13) rectangle (7.4,-4.57);
\filldraw[fill=satpurple!17,draw=satpurple] (7.4,-4.13) rectangle (12,-4.57);
\node at (7,-3.05) {a\quad a\quad a\quad a\quad a\quad a\quad a\quad a\quad a\quad a\quad a\quad a};
\node at (8.35,-3.7) {All identical};
\node at (9.7,-4.35) {All identical};
\node at (3.35,-3.7) {May differ};
\node at (4.7,-4.35) {May differ};
\draw[->,thick] (2,-4.95) -- (12,-4.95);
\node[anchor=north] at (7,-5.02) {Position in a run of the same token};
\node[anchor=north,font=\footnotesize\bfseries] at (6,-5.55) {Eventually, one more repeated token leaves the next-token prediction unchanged.};
\end{tikzpicture}
\caption{How ECB transformers lose track of repetition through rounded accumulation.
Panel (a) shows the left-to-right accumulation order. In panel (b), colored tails indicate guaranteed identical states within each layer; gray states may differ (schematic).}
\label{fig:layerwise_saturation}
\end{figure}

\begin{lemma}[Terminal-run saturation]
\label{lem:saturated_suffix_self_loop}
Fix a horizon~$N$ and $T_N\in\mathcal T_{\mathrm{ECB}}$ with depth $D=D(N)$.
Let
\[
  \kappa_D
  :=
  \sum_{\ell=1}^{D}
  \left(\bar n_\ell+2^{p(N)}\right),
  \qquad
  \bar n_\ell:=\max_h n_h^{(\ell)} .
\]
For every fixed prefix $x\in\Gamma^*$, every $\sigma\in\mathcal V$, and every $k,k'\ge\kappa_D$ for which the contexts \((x,\sigma^k)\) and \((x,\sigma^{k'})\) are in the model domain, the next-token logits on these two contexts are \textbf{identical}.
\end{lemma}

It follows that, after a long enough terminal run of $\sigma$'s, the next-token map may still depend on the earlier prefix~$x$, but it no longer depends on the length of that run.

\begin{restatable}[$\mathcal T_{\mathrm{ECB}}$ lower bound: correct copy implies a $\sigma$-loop]{theorem}{impossibilitythm}
\label{thm:impossibility}
Assume the deterministic left-to-right finite-precision execution model of \Cref{sec:float}.
Let $\{T_N\}\in\mathcal T_{\mathrm{ECB}}$ have depth $D(N)$ and precision $p(N)$ at design horizon~$N$.
Let $\kappa_D(N)$ be the saturation threshold from Lemma~\ref{lem:saturated_suffix_self_loop} for~$T_N$.
If $\kappa_D(N)=o(N)$, then for all sufficiently large $N$, every $L\in[N/3,N/2]$, and every fixed $\sigma\in\mathcal V$, $T_N$ fails on the all-identical prompt
\[
  \langle\BoS\rangle,\sigma^L,\langle\EoI\rangle,\langle\Copy\rangle .
\]
More precisely, it either:
\begin{enumerate}
    \item Makes a mistake before completing the $L$-token copy, or,
    \item if it outputs $y_1=\cdots=y_L=\sigma$, then $y_r = \sigma$ for every $r>L$ for which the context remains in the model domain of $T_N$.
\end{enumerate}
\end{restatable}

\paragraph{Proof idea.}
Let $x_L:=P(\sigma^L)$ be the prompt for the all-identical source.
Since $\kappa_D(N)=o(N)$, every $L\in[N/3,N/2]$ satisfies $L-1\ge\kappa_D(N)$ for sufficiently large~$N$.
If the $L$-th copied token is correctly generated from \((x_L,\sigma^{L-1})\), then the termination query \((x_L,\sigma^L)\) has the same logits and hence also outputs~$\sigma$, even though the correct token is $\langle\EoS\rangle$.
Repeating the same argument gives the loop.

For the recursive ALiBi case below, the eventual-constancy argument assumes that the bias recurrence does not overflow, as in Lemma~\ref{lem:alibi_ecb}.
This isolates stagnation from rounding rather than failure from exponent overflow.

\begin{restatable}{corollary}{maincorollary}
\label{cor:main}
Consider the variants of ALiBi~\citep{press2022alibi} and Hard-ALiBi~\citep{jelassi2024repeat} as defined in \Cref{def:alibi_hard_alibi}.
For depth-$D(N)$, arbitrary-width, deterministic $p(N)$-bit models satisfying $D(N)2^{p(N)}=o(N)$, Theorem~\ref{thm:impossibility} applies to:
\begin{itemize}[leftmargin=*,nosep]
\item \textbf{NoPE:} no positional logit bias.
\item \textbf{Recursive implementation of ALiBi:} the rounded slope recurrence $b_h(d+1)=\mathsf{fl}(b_h(d)+m_h)$.
\item \textbf{Bounded-window Hard-ALiBi:} hard-masked heads whose attention windows are at most $2^{p(N)}$, with any remaining heads treated as NoPE heads.
\end{itemize}
% The recursive implementation of ALiBi and bounded-window Hard-ALiBi are defined in Appendix~\ref{app:theory_proofs}, Definition~\ref{def:alibi_hard_alibi}.
In particular, this applies even when $D=N^{0.01}$ and $p\le0.9\log_2N$.
% Consequently, none of these classes can solve horizon-uniform copying: on each sufficiently long all-identical source $\sigma^L$, the model either fails before completing the $L$-token copy or, after completing it, continues outputting~$\sigma$ instead of terminating.
\end{restatable}

% The proof is immediate from  heorem~\ref{thm:impossibility}: NoPE has zero distance bias, the recursive implementation of ALiBi becomes eventually constant after at most~$2^{p(N)}$ finite-precision states, and bounded-window Hard-ALiBi is constant outside its window.
% Thus $\kappa_D(N)\le2D(N)\,2^{p(N)}$, and the sufficient tradeoff $D(N)2^{p(N)}=o(N)$ gives the corollary.
% Appendix~\ref{app:theory_proofs} verifies the ALiBi and Hard-ALiBi cases in detail.

\paragraph{Scope.}
Under exact accumulation, the attention mass assigned to a repeated run relative to the fixed prefix can continue to carry length information.
The lower bound identifies how repeated rounding in the execution model of \Cref{sec:float} makes this information stop changing.
Beyond $\mathcal T_{\mathrm{ECB}}$, we also observe copying failures on long all-identical strings within the training range for RoPE, sinusoidal embeddings, and directly computed ALiBi (\Cref{subsec:standard_positional_experiments}).

%% ==================================================================
\subsection{Upper Bound via VO Position Hashing}
\label{sec:construction}

The positive construction replaces ``content hashing'' in \cite{jelassi2024repeat} by position hashing.
Layer~1 selects a delimiter rather than a repeated source token and writes the query's distance from that delimiter into the value stream.
Layer~2 uses this representation to select the next source occurrence, without matching token content.

\paragraph{Anchors and offsets.}\label{sec:anchors}
We call \(\langle\BoS\rangle\) and \(\langle\Copy\rangle\) the \textit{anchors}.
The \textit{active anchor} $a(t)$ is the position of \(\langle\Copy\rangle\) if it has appeared by time~$t$, and the position of \(\langle\BoS\rangle\) otherwise.
The \textit{active-anchor offset} is $\delta(t):=t-a(t)$.
Source token~$s_k$ has offset~$k$, and the source-side \(\langle\EoI\rangle\) has offset~$L+1$.
At a copy-region query of offset~$k$, the desired source-side offset is~$k+1$.
All visible copy-region offsets are at most~$k$, so this target \textit{occurs only on the source side}: it holds $s_{k+1}$ if $k<L$ and \(\langle\EoI\rangle\) if $k=L$.
After the second layer, a fixed linear readout maps each vocabulary token to itself and the retrieved \(\langle\EoI\rangle\) to \(\langle\EoS\rangle\).
% This precisely measures the \textit{offset} $k$ for each source token / correctly copied prefix token $s_k$. 
\paragraph{Writing offsets into values.}
For period $q$, let $R_q: \mathbb{R}^2 \to \mathbb{R}^2$ be the rotation by an angle of $2\pi/q$.

\begin{definition}[VO-rotated attention head]\label{def:vo_rotated_head}
A VO head of period~$q$ uses ordinary NoPE attention weights $\alpha_{t,j}$, obtained by applying softmax to the dot products $\mathbf q_t^\top\mathbf k_j$ over $j\le t$.
The head selects some value coordinates as ``rotary coordinates'' and groups them into pairs. For each pair, let $\mathbf v_j \in \mathbb{R}^2$ be the corresponding vector at key position~$j$.
In exact arithmetic, the pair's output at query position~$t$ is $R_q^t\sum_{j\le t}\alpha_{t,j}R_q^{-j}\mathbf v_j$.
Other non-rotary value coordinates are aggregated without rotation.
\end{definition}

Each Layer-1 VO head in our construction uses only a single pair of rotary coordinates, with value $\mathbf e_1$ at every key position.
When a query at~$t$ attends one-hot to the anchor at~$j$, the value output is $R_{q_h}^{t-j}\mathbf e_1$, whose phase records the offset modulo~$q_h$. We call the concatenation of these phase vectors the ``offset fingerprint''.
Lemma~\ref{lem:fingerprint} accounts for the error from finite-precision rotations.

\begin{definition}[VO positional transformer class $\mathcal T_{\mathrm{VO}}$]\label{def:vo}
We write $\mathcal T_{\mathrm{VO}}$ for transformers equipped with such VO heads, where each position has exact access to its index modulo each chosen period to select rounded rotation-table entries; Appendix~\ref{app:vo_formal} gives the formal instantiation.
\end{definition}

\begin{restatable}[$\mathcal T_{\mathrm{VO}}$ upper bound with CRT-chosen periods]{theorem}{constructionthm}
\label{thm:construction}
For every fixed finite vocabulary~$\mathcal V$, there exists a family $\{T_N\}\subset\mathcal T_{\mathrm{VO}}$ of two-layer transformers such that, for all sufficiently large design horizons~$N$, $T_N$ correctly copies every string of every length $L\le N$:
\[
  \langle\BoS\rangle,s_1,\ldots,s_L,\langle\EoI\rangle,\langle\Copy\rangle
  \quad\mapsto\quad
  s_1,\ldots,s_L,\langle\EoS\rangle .
\]
The model $T_N$ uses:
\begin{itemize}[leftmargin=*,nosep]
\item $H=\Theta(\log N/\log\log N)$ first-layer prime-period heads, each with two rotary value coordinates;
\item floating-point precision $p=O(\log\log N)$;
\item dimension $d=O(|\mathcal V|+\log N/\log\log N)$.
\end{itemize}
\end{restatable}

\paragraph{Proof sketch.}
Layer~1 attends one-hot to the active anchor: \(\langle\BoS\rangle\) before \(\langle\Copy\rangle\), and \(\langle\Copy\rangle\) afterward.
VO rotation encodes $\delta(t)\bmod q_h$ into head~$h$.
We choose $H$ such that the product of the first $H$ primes exceeds $N+1$, which allows us to distinguish all offsets up to $N+1$.
At a copy-region query of offset~$k$, Layer~2's query projection applies a fixed one-step rotation to each fingerprint component, in order to encode the target offset~$k+1$, while the keys keep the fingerprint at each position unchanged.
Before scaling and rounding, their centered inner products would sum to
\[
  \sum_{h=1}^{H}\left[\cos\!\left(\frac{2\pi(\delta(j)-\delta(t)-1)}{q_h}\right)-1\right].
\]
The target at offset~$k+1$ matches every residue, and scores zero. By the Chinese Remainder Theorem, every other visible position mismatches in at least one term, which contributes at most $-8/q_h^2$; the other terms cannot offset this penalty because they are all nonpositive.
Since the largest period is $\Theta(\log N)$, the gap survives rounding at $O(\log\log N)$-bit precision. We scale the scores so that, after subtracting the maximum in softmax, all non-target exponentials round to zero. This makes attention exactly one-hot on the target.
Appendix~\ref{app:theory_proofs} gives the finite-precision analysis.
For each prime period~$q_h$, the position generator can maintain the exact phase index $c_{h,t}=t\bmod q_h$ with an $O(\log\log N)$-bit counter updated by $c_{h,t+1}=(c_{h,t}+1)\bmod q_h$, and use it to look up the rounded sine and cosine entries for the rotation.
Appendix~\ref{app:vo_implementation} gives the storage costs of the phase tables and counter state.

%% ==================================================================
% The difficulty is not the fixed-context next-token map itself.
% Appendix~\ref{app:ac0} shows that, at every fixed context length, autoregressive copying is computed by a $\mathrm{DLOGTIME}$-uniform $\mathrm{AC}^0$ delimiter-indexed multiplexer.
% Such a circuit can compile the length-specific table $c\mapsto n-c+2$ into its wiring.
% The horizon-uniform transformer question is whether one parameter setting for a design horizon~$N$ represents the needed source-position information for all $L\le N$ when content aliases many positions.

% !TEX root = ../main.tex
%% ===========================================================================
%% SECTION 4: EXPERIMENTS
%% ===========================================================================
\section{Experiments}
\label{sec:experiments}

Figure~\ref{fig:llm_copying} shows that frontier LLMs which copy random strings accurately fail on repeated strings. For example, Qwen3.5 produced a $79{,}984$-character answer when asked to copy $300$ \texttt{b}'s. Appendix~\ref{app:llm_experiment} provides further examples, as well as prompting and scoring details. In addition, we also use synthetic experiments to test how standard positional baselines and VO-Rotation learn to copy under repetition and generalize beyond the training lengths.\footnote{Experimental code and recorded LLM responses are available at \url{https://github.com/bobeatslollipop/copying-under-repetition}.}

\subsection{Synthetic Setup}
\label{subsec:exp_design}

Each example uses the copying task of \Cref{sec:setup} over the $26$ lowercase English letters; we write $n$ for the source length, denoted $L$ there.
The loss is applied only to the generated answer, including the end marker.
At each answer position, we take the next-token under greedy/argmax decoding, given the correct prefix (teacher forcing). A copy counts as exact only if every prediction matches, including the end marker.
% The theory separates
% \(\langle\EoI\rangle\) and \(\langle\EoS\rangle\), while the implementation uses
% a single period token; this only changes the final fixed output transduction.
During training, the source length is uniform on $\{1,\ldots,2000\}$.
We train on a mixture of i.i.d.\ strings ($10\%$), all-identical strings ($10\%$), and strings of same-character blocks with varying repetition levels ($80\%$).
We train each model with seeds $42$, $43$, and $44$, using a shared configuration within each model family. Architectures and training budgets differ across model families.
\Cref{app:synthetic_experiment_details} records the configurations, block sampler, literal delimiter encoding, and evaluation details.

Evaluation uses uniform-random strings ($p_{\rm switch}=1$), moderately repeated strings ($p_{\rm switch}\approx0.01$), and all-identical strings ($p_{\rm switch}\to0$).
% The three evaluation distributions are
% held-out diagnostics for different amounts of content ambiguity, not three
% separately tuned training distributions.

\subsection{Standard Positional Baselines}
\label{subsec:standard_positional_experiments}

We first compare NoPE, RoPE, ALiBi, Hard-ALiBi, and
Sinusoidal~\citep{Vaswani+2017} using the same $6$-layer, width-$512$, $8$-head \texttt{bf16}
backbone.  We compute the ALiBi bias directly as $-m_h(i-j)$, not by the rounded recurrence in Corollary~\ref{cor:main}.
The contrast is sharpest near the training limit: every standard positional baseline has mean accuracy below $0.26$ on the all-identical diagnostic, while the two-layer VO model reaches $1.000$ in all three seeds (Table~\ref{tab:synthetic_accuracy_1993} in Appendix~\ref{app:synthetic_experiment_details}).
Cross-seed variation is substantial for some baselines: RoPE's all-identical accuracy is $0.257\pm0.367$ (mean $\pm$ sample standard deviation).
Figure~\ref{fig:synthetic_standard_pe} shows how accuracy changes with length.
% The narrower point is that
% mechanisms that affect attention selection do not by themselves write a
% reusable selected-offset signal into the residual stream.

\begin{figure}[t]
    \centering
    \includegraphics[width=0.94\textwidth]{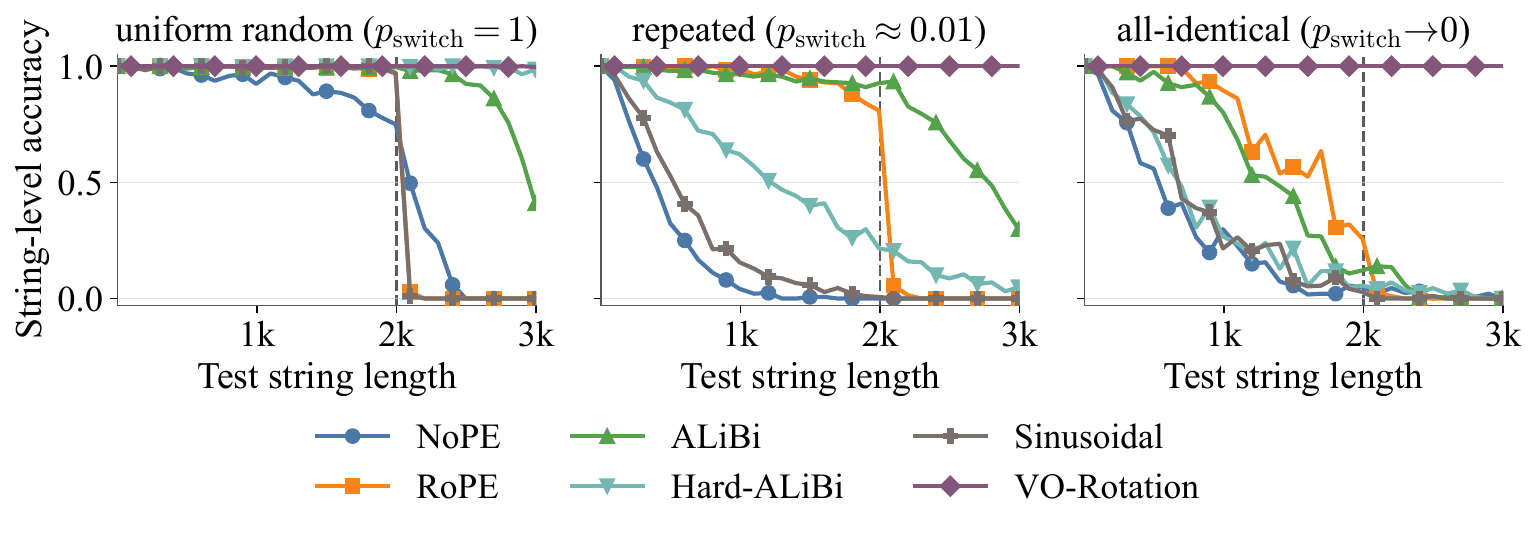}
    \caption{Exact-copy accuracy by source length and distribution, averaged over the three training seeds for every model. Standard-deviation bands are omitted for readability; Table~\ref{tab:synthetic_accuracy_1993} reports cross-seed variation near the training limit. The dashed black line marks the maximum training source length, $n=2000$; lengths to its right test extrapolation.}
    \label{fig:synthetic_standard_pe}
\end{figure}

\subsection{VO-Rotation}
\label{subsec:vo_experiments}

We next test whether training can exploit the value-side positional encoding.
The VO-Rotation models use a $2$-layer, width-$256$, $4$-head backbone with no score-side positional term.
Following the construction's layout, the first layer applies value/output rotation with one prime period per head ($5,7,11,13$), and the second layer uses ordinary NoPE attention.
The joint residue vector repeats every $5005$ offsets.
We simulate FP8 by rounding linear-layer inputs and weights; other operations use native PyTorch mixed precision (Appendix~\ref{app:synthetic_experiment_details}).

\begin{figure}[!t]
    \centering
    \includegraphics[width=0.94\textwidth]{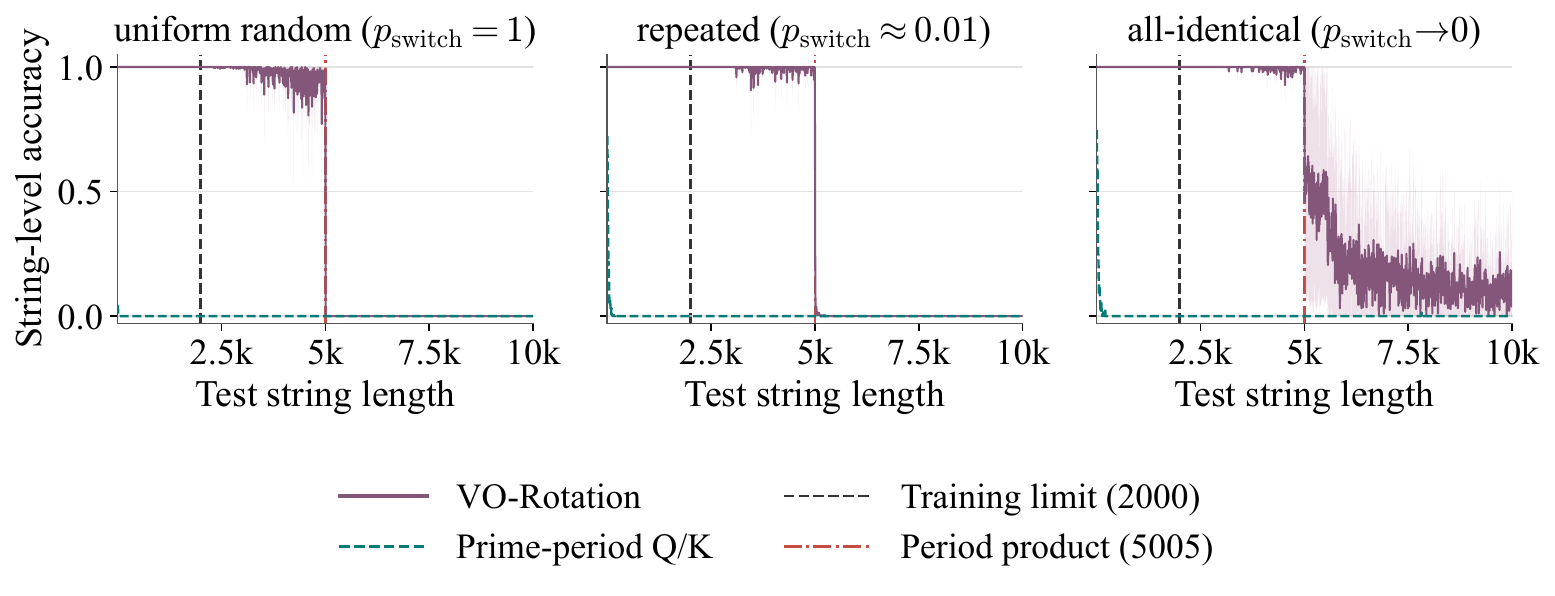}
    \caption{Teacher-forced exact-copy accuracy (including EOS) on the common length grid. Curves show three-seed means; the VO band shows $\pm1$ sample standard deviation, clipped to $[0,1]$. Vertical lines mark the training limit ($2000$) and period product ($5005$).}
    \label{fig:synthetic_vo_generalization_10000}
\end{figure}

\paragraph{Length generalization and the period product.}
Every VO seed achieves perfect accuracy at all evaluated lengths through $2000$ in all three conditions (\Cref{fig:synthetic_vo_generalization_10000}).
Over extrapolation lengths $2008$--$5000$, mean accuracy across lengths and seeds is $0.980$ for uniform-random, $0.996$ for moderately repeated, and $0.998$ for all-identical strings.
Across these extrapolation lengths, every seed averaged at least $95.9\%$ exact-copy accuracy in every condition (Appendix~\ref{subsec:vo_generalization_10000}).

To test whether the same prime periods are as effective on the score side, we move the Layer~1 rotations from V/O to Q/K in a prime-period Q/K control. The two-layer backbone, FP8 simulation, training data distribution, periods $(5,7,11,13)$, and rotary dimensions stay fixed.
The Q/K runs finish at $195{,}000$ optimizer updates per seed, compared with $148{,}200$ to $158{,}200$ for VO.
At the training limit $n=2000$, teacher-forced exact-copy accuracy is $100\%$ for VO and $0\%$ for Q/K in every condition and seed (\Cref{fig:synthetic_vo_generalization_10000}).
On the all-identical strings at that length, Q/K predicts every copied character correctly but never predicts EOS at the target stopping position (\Cref{subsec:prime_qk_control}), a failure pattern resembling those observed in frontier LLMs and predicted by our theory for the ECB class.

For VO, uniform-random and moderately repeated accuracy drops sharply just beyond the period product $5\cdot7\cdot11\cdot13=5005$ in every seed.
Offsets separated by $5005$ share residues, so residue-only retrieval cannot distinguish them.
We conclude that the VO models have likely learned to utilize the periodic value-side positions for copying.

Beyond the period product, VO accuracy on all-identical strings varies substantially across seeds (Appendix~\ref{subsec:vo_generalization_10000}).
All-identical strings can still be copied after residue collisions: confusing two letter positions preserves the copied character, whereas confusing a letter with the end delimiter can cause a termination error.

% The theorem uses an
% idealized residue mechanism with exact period-indexed powers, chosen logit
% gaps, and deterministic finite-precision routing.  The \texttt{fp8} run is a
% finite-\(N\) trainability diagnostic for a differentiable value-side offset
% channel.  Because \texttt{fp8} is fixed-width low precision, it should
% not be interpreted as an asymptotic precision claim.

% !TEX root = ../main.tex
%% ===========================================================================
%% SECTION 5: CONCLUSION
%% ===========================================================================
\section{Conclusion}
\label{sec:conclusion}

Repetition makes content-based occurrence matching ambiguous and motivates
representing positions relative to delimiters.
VO-rotation writes anchor-relative offsets into the value stream. With
prime-period phases and CRT, a two-layer model copies arbitrary
strings throughout its design horizon, using the same retrieval rule for
copying and termination at $O(\log\log N)$-bit precision.
The $\mathcal T_{\mathrm{ECB}}$ lower bound identifies a width-independent
finite-precision obstruction: under the stated saturation condition, a model
cannot both finish a sufficiently long identical copy and terminate correctly.
In the synthetic experiments, two-layer VO models learn to copy with near-perfect
accuracy up to length $5000$, $2.5$ times the maximum training length.
On random and moderately repeated strings, their accuracy collapses just beyond
the period product $5005$, as expected when offset residues collide.

The broader design question is how positional information should reach later
layers under repetition. Hybrid VO/RoPE training and the effect of
value-side rotations on tasks that rely on value semantics remain to be tested.
On the theory side, it remains open whether the saturation argument extends to
other accumulation rules, and whether a score-side mechanism given the same
residues can copy with comparable precision and depth.

\section*{AI Disclosure}

We used GPT, Gemini, and Claude to assist with proof development, experiment implementation, figure preparation, and writing.
The authors checked the proofs and references, ran the code, reviewed the results and final text, and take responsibility for the final manuscript.

% !TEX root = ../main.tex
\section*{Acknowledgments}
Anxin Guo is supported by NSF grants ECCS-2216970 and CCF2154100 through Aravindan Vijayaraghavan.
Zhiyuan Li is supported by DARPA AIQ Program and OpenAI superalignment award. 

\bibliography{ref}
\bibliographystyle{iclr2027_conference}

\appendix
% !TEX root = ../main.tex
\section{Details of the LLM Copying Experiment}
\label{app:llm_experiment}

\subsection{Prompt and Experimental Setup}
Each model is instructed via the following system prompt:

\begin{quote}
    \textit{``You are a precise string-copying assistant. You will see a string enclosed between '\texttt{\$}' and '\texttt{.}', followed by '\texttt{:}'. Output ONLY the enclosed string, character by character, with absolutely nothing else -- no spaces before or after, no explanation, no newline.''}
\end{quote}

The user prompt has the format
\[
    \texttt{\$} \quad \langle \text{string} \rangle \quad \texttt{.}
    \quad \texttt{:}
\]
where \texttt{.} is the prompt end delimiter and \texttt{:} is the copy cue.
The model is tasked to output exactly the enclosed string, not the delimiters.

\paragraph{Lengths and repetition conditions.}
We evaluate source strings of $30$, $100$, $300$, $1000$, and $3000$ characters under three conditions: uniform-random lowercase strings, moderately repeated strings built from geometric same-character blocks, and all-identical strings.
For each block, we independently sample a lowercase letter uniformly and a length from a geometric distribution on the positive integers with success probability $p_{\rm switch}=1/101\approx0.01$.
We truncate the final block at the requested source length.
Adjacent sampled blocks may share a character.

\paragraph{Output extraction and scoring.}
Accuracy is the fraction of all $50$ trials per length and condition whose extracted prediction exactly matches the target.
The extraction routine removes \texttt{<think>} blocks, splits the remaining response into nonempty lines, and removes answer prefixes, delimiter symbols, quotes, backticks, and spaces.
Among cleaned lines containing only lowercase letters, it selects the one whose length is closest to the target length; if none exists, it uses the cleaned last line.
Empty outputs count as copying failures.

\paragraph{Model variants.}
Table~\ref{tab:llm_variant_provenance} identifies the model IDs and reasoning settings for each evaluated variant.

Qwen uses Ollama's chat API, GPT-5.5 uses the OpenAI Responses API, Gemini uses the Gemini API, and DeepSeek uses the DeepSeek API.
The Qwen, Gemini, and DeepSeek requests set temperature to $0$ and top-$p$ to $1$; Qwen additionally sets top-$k$ to $1$ and the repetition penalty to $1$.
OpenAI requests specify the reasoning effort and output budget but leave temperature and top-$p$ unset.

\begin{table}[!htbp]
    \centering
    \footnotesize
    \setlength{\tabcolsep}{3.4pt}
    \caption{Frontier-model variants plotted in \Cref{fig:llm_copying}.  Each
    row uses $50$ trials for each length and condition, a maximum generation
    budget of $20{,}000$ output tokens, and the extraction routine
    described above.}
    \label{tab:llm_variant_provenance}
    \begin{tabular}{p{0.25\linewidth}p{0.36\linewidth}p{0.30\linewidth}}
        \toprule
        Model & Model ID & Reasoning/thinking \\
        \midrule
        Qwen3.5 &
        \texttt{\seqsplit{qwen3.5:397b-cloud}} &
        thinking disabled \\
        DeepSeek V4 Pro &
        \texttt{\seqsplit{deepseek-v4-pro}} &
        thinking disabled \\
        Gemini 3.1 Pro &
        \texttt{gemini-3.1-pro-}\newline\texttt{preview} &
        low thinking \\
        Gemini 3.1 Pro &
        \texttt{gemini-3.1-pro-}\newline\texttt{preview} &
        medium thinking \\
        GPT-5.5 &
        \texttt{\seqsplit{gpt-5.5}} &
        none \\
        GPT-5.5 &
        \texttt{\seqsplit{gpt-5.5}} &
        medium reasoning \\
        \bottomrule
    \end{tabular}
\end{table}

\subsection{Copying Failures in Frontier LLMs}
\label{subsec:llm_failure_examples}

Table~\ref{tab:failure_modes} gives three recorded Qwen3.5 outputs with thinking disabled for all-identical sources of length $300$: extended repetition, premature truncation, and an empty output.
Across all lengths, Qwen3.5 returns $112$ empty outputs: $43$ on moderately repeated strings and $69$ on all-identical strings, with none on uniform-random strings.

\begin{table}[!htbp]
    \centering
    \small
    \caption{Recorded Qwen3.5 copying failures on all-identical sources of length $300$. String excerpts reproduce the recorded text; ellipses mark omitted characters.}
    \label{tab:failure_modes}
    \begin{tabular}{@{}p{0.18\textwidth}p{0.32\textwidth}p{0.42\textwidth}@{}}
        \toprule
        Failure type & Target string (excerpt) & Model reply (excerpt) \\
        \midrule
        Extended\newline repetition &
        \texttt{\seqsplit{bbbbbbbbbbbbbbbbbbbbbbbbbbbbbbbbbbbbbbbbbbbbbbbb}}\(\ldots\)\newline
        (Length: $300$) &
        \texttt{\seqsplit{bbbbbbbbbbbbbbbbbbbbbbbbbbbbbbbbbbbbbbbbbbbbbbbbbbbbbbbbbbbbbbbb}}\(\ldots\)\newline
        (Length: \textbf{79,984}) \\
        \midrule
        Premature\newline truncation &
        \texttt{\seqsplit{iiiiiiiiiiiiiiiiiiiiiiiiiiiiiiiiiiiiiiiiiiiiiiii}}\(\ldots\)\newline
        (Length: $300$) &
        \texttt{\seqsplit{iiiiiiiiiiiiiiiiiiiiiiiiiiiiiiiiiiiiiiiiiiiiiiiiiiiiiiiiiiiiiiii}}\(\ldots\)\newline
        (Length: \textbf{224}) \\
        \midrule
        Empty output &
        \texttt{\seqsplit{zzzzzzzzzzzzzzzzzzzzzzzzzzzzzzzzzzzzzzzzzzzzzzzz}}\(\ldots\)\newline
        (Length: $300$) &
        \textit{Empty response}\newline
        (Length: $0$) \\
        \bottomrule
    \end{tabular}
\end{table}

\subsection{Repetition as a Positional Stress Test}
Frontier LLMs copy random strings accurately but fail on repeated strings, where local content no longer distinguishes occurrences.
The all-identical condition also isolates the challenge studied in Theorem~\ref{thm:impossibility}: the model must stop at the correct length even when all source characters are the same.

Repetition in characters can also produce repetition in tokens.
Under OpenAI's \texttt{o200k\_base} tokenizer, a run of $300$ \texttt{b}'s becomes $75$ copies of the token \texttt{bbbb}, and a run of $302$ becomes the same $75$ copies followed by one \texttt{bb}.
In this example, the tokenizer represents the run length by the number of repeated tokens and any shorter token at the end.

% !TEX root = ../main.tex
\section{Additional Synthetic Experiment Details}
\label{app:synthetic_experiment_details}

\paragraph{Task encoding and string generation.}
The synthetic implementation uses the literal tokens \(\texttt{\$}\), \(\texttt{;}\), \(\texttt{:}\), and \(\texttt{.}\) for \(\langle\BoS\rangle\), \(\langle\EoI\rangle\), \(\langle\Copy\rangle\), and \(\langle\EoS\rangle\), respectively:
\[
    \texttt{\$}, x_1,\ldots,x_n,\texttt{;},\texttt{:}
    \;\mapsto\;
    x_1,\ldots,x_n,\texttt{.}.
\]
The loss and exact-copy metric cover the generated answer after \(\texttt{:}\), including the final period token.
Training draws $n\sim\mathrm{Unif}\{1,\ldots,2000\}$ and uses uniform-random, all-identical, and same-character-block strings with probabilities $0.1$, $0.1$, and $0.8$, respectively.
For each block string, we draw the mean block length $1/p_{\rm switch}$ uniformly from $[1,301]$.
For each block, we independently draw a character uniformly from the alphabet and a length from $\mathrm{Geom}(p_{\rm switch})$ on the positive integers, truncating the last block to reach source length~$n$.
Adjacent blocks share the same character with probability $1/26$, so a maximal same-character run may span several sampled blocks.
Moderately repeated evaluation strings use $p_{\rm switch}=1/101\approx0.01$, a mean block length of $101$.

\paragraph{Training configurations.}
Tables~\ref{tab:synthetic_run_details} and~\ref{tab:baseline_seed_runs} summarize the training configurations and optimizer updates used in \Cref{sec:experiments}.
Optimization uses AdamW with $(\beta_1,\beta_2)=(0.9,0.99)$, $\epsilon=10^{-8}$, learning rate $10^{-4}$, and no weight decay.

All models use effective batch size~$32$ (microbatch~$4$
with gradient accumulation~$8$ for the \texttt{bf16} baselines, microbatch~$8$
with gradient accumulation~$4$ for VO-Rotation and its Q/K control),
GELU MLPs, layer normalization, and warmup-stable-decay learning-rate scheduling~\citep{hu2024minicpm} with $2000$ warmup
optimizer updates, followed by a constant-learning-rate phase and a final linear decay to zero (cooldown).

The standard positional baselines use $8$ attention heads
of dimension $64$.  The ALiBi baseline computes the distance bias directly, with
slopes from the \href{https://huggingface.co/mosaicml/mpt-7b/blob/main/attention.py}{MPT-7B implementation}: for $8$ heads, the slopes are
\(m_h=2^{-h}\), \(h=1,\ldots,8\), so up to row-wise softmax constants the bias
for query position \(i\) and key position \(j\le i\) is \(-m_h(i-j)\).
Hard-ALiBi masks four heads: for head
\(h=1,\ldots,4\), only keys with backward offset \(i-j\in\{0,\ldots,h-1\}\)
remain visible, while heads $5$--$8$ keep the ordinary causal mask.
In each VO model, the first layer has four VO heads with periods $5$, $7$, $11$, and $13$, one period per head.
This follows the one-period-per-head structure of \Cref{thm:construction}, but each head's rotary value subspace has $32$ of its $64$ coordinates ($16$ two-dimensional blocks) instead of the theorem's single block, and the value vector in this subspace is learned
rather than fixed to $\mathbf e_1$.  All rotary blocks within a head share
the head's frequency $2\pi/q_h$, where $q_h$ is its period, so the extra width provides learned value components at that frequency rather than additional periods.

All budgets and step counts are in optimizer updates.
We monitor training loss every $200$ updates using an exponential moving average with decay $0.98$.
For VO, the average is applied to the loss, and cooldown starts when it falls below $10^{-4}$.
For the standard baselines, it is applied to the base-$10$ log loss; from update $50{,}000$ onward, cooldown starts if this average has not improved by at least $0.05$ over $12{,}500$ updates.
If cooldown starts at update $t$, the learning rate decays linearly to zero over $\max\{10{,}000,\lfloor0.3t\rfloor\}$ updates, after which training stops.
Cooldown also starts when $t+\lfloor0.3t\rfloor$ reaches the configured update budget.

\paragraph{Precision simulation.}
We simulate FP8 forward passes by quantizing each linear layer's input and weight, then returning both to their incoming storage dtypes before the linear operation.
Quantization uses PyTorch's \texttt{float8\_e4m3fn} cast without tensor-wise or channel-wise scaling or extra clipping, with fraction-bit truncation when the cast is unavailable.\footnote{The fallback retains the incoming exponent range and keeps three fraction bits for \texttt{float32}/\texttt{bfloat16} operands and one for \texttt{float16} operands.}
We use straight-through gradients and retain master parameters, linear biases, and optimizer states at their original precision.
Training runs under Hugging Face Accelerate, which selects CUDA autocast in \texttt{bf16} on GPUs that support it, as the RTX A4000 does, and in \texttt{fp16} otherwise.
Attention matrix products, softmax, residual updates, and layer normalization use native PyTorch mixed-precision arithmetic.
Position indices are integers, while phase multiplication and sine/cosine evaluation use \texttt{float32}.

\paragraph{Evaluation metrics.}
Section~\ref{sec:experiments} reports string-level exact-copy accuracy using the final checkpoint, averaged over $96$ examples per length and training seed.
The baseline evaluations use source lengths
\(1,9,17,\ldots,2993\); the VO length-generalization evaluation uses
\(8,16,\ldots,10{,}000\). Figure~\ref{fig:synthetic_standard_pe} uses a sparse subset of baseline lengths.
Because the VO grid differs, its curve in that figure uses, at each displayed baseline length, the accuracy at the nearest evaluated VO length, at most four tokens away.
Table~\ref{tab:synthetic_accuracy_1993} reports the means and sample standard deviations near the training limit.

\begin{table}[tbp]
    \centering
    \caption{String-level accuracy near the $n=2000$ training limit: mean $\pm1$ sample standard deviation across the three training seeds. Baselines are evaluated at $n=1993$ and VO at $n=1992$.}
    \label{tab:synthetic_accuracy_1993}
    \small
    \setlength{\tabcolsep}{3pt}
    \begin{tabular}{p{0.24\linewidth}ccccc}
        \toprule
        Model & Layers & Width & \(p_{\rm switch}=1\) & \(p_{\rm switch}=0.01\) & \(p_{\rm switch}\to0\) \\
        \midrule
        NoPE & 6 & 512 & $0.750\pm0.244$ & $0.000\pm0.000$ & $0.049\pm0.084$ \\
        RoPE & 6 & 512 & $0.972\pm0.030$ & $0.809\pm0.071$ & $0.257\pm0.367$ \\
        ALiBi & 6 & 512 & $0.997\pm0.006$ & $0.927\pm0.038$ & $0.122\pm0.080$ \\
        Hard-ALiBi & 6 & 512 & $1.000\pm0.000$ & $0.215\pm0.129$ & $0.056\pm0.006$ \\
        Sinusoidal & 6 & 512 & $0.969\pm0.021$ & $0.007\pm0.006$ & $0.028\pm0.039$ \\
        \midrule
        VO-Rotation & 2 & 256 & $1.000\pm0.000$ & $1.000\pm0.000$ & $1.000\pm0.000$ \\
        \bottomrule
    \end{tabular}
\end{table}

\paragraph{Compute.}
Each run used one GPU at a time.
The standard positional-encoding baselines used NVIDIA RTX A6000 GPUs; all three VO-Rotation training runs used NVIDIA RTX A4000 GPUs.
The prime-period Q/K runs used RTX A6000 GPUs, except that seed~$44$ started on an RTX A4000 before resuming on an A6000.

\begin{table}[t]
    \centering
    \small
    \setlength{\tabcolsep}{3.2pt}
    \caption{Training configurations shared by seeds $42$, $43$, and $44$. Budgets are maximum optimizer updates. FP8 sim. denotes the linear-operand simulation described above.}
    \label{tab:synthetic_run_details}
    \begin{tabular}{lccccc}
        \toprule
        Model & Layers & Width & Heads & Precision & Update budget \\
        \midrule
        NoPE & 6 & 512 & 8 & \texttt{bf16} & $320{,}000$ \\
        RoPE & 6 & 512 & 8 & \texttt{bf16} & $320{,}000$ \\
        ALiBi & 6 & 512 & 8 & \texttt{bf16} & $320{,}000$ \\
        Hard-ALiBi & 6 & 512 & 8 & \texttt{bf16} & $320{,}000$ \\
        Sinusoidal & 6 & 512 & 8 & \texttt{bf16} & $320{,}000$ \\
        VO-Rotation & 2 & 256 & 4 & FP8 sim. & $400{,}000$ \\
        Prime-period Q/K & 2 & 256 & 4 & FP8 sim. & $195{,}000$ \\
        \bottomrule
    \end{tabular}
\end{table}

\begin{table}[t]
    \centering
    \small
    \caption{Completed optimizer updates for each model and training seed, using the shared configurations in Table~\ref{tab:synthetic_run_details}.}
    \label{tab:baseline_seed_runs}
    \begin{tabular}{lrrr}
        \toprule
        Model & Seed $42$ & Seed $43$ & Seed $44$ \\
        \midrule
        NoPE & $102{,}800$ & $115{,}440$ & $110{,}760$ \\
        RoPE & $144{,}800$ & $112{,}580$ & $119{,}080$ \\
        ALiBi & $81{,}200$ & $81{,}120$ & $118{,}560$ \\
        Hard-ALiBi & $110{,}200$ & $96{,}980$ & $114{,}140$ \\
        Sinusoidal & $113{,}800$ & $119{,}080$ & $104{,}520$ \\
        VO-Rotation & $158{,}200$ & $150{,}020$ & $148{,}200$ \\
        Prime-period Q/K & $195{,}000$ & $195{,}000$ & $195{,}000$ \\
        \bottomrule
    \end{tabular}
\end{table}

\subsection{Prime-Period Q/K Control}
\label{subsec:prime_qk_control}

To test rotation placement, we keep the VO model's per-head periods but rotate queries and keys instead of values and outputs. In Layer~1, we rotate $32$ of the $64$ coordinates in each query and key vector. Layer~2 is unchanged.
Q/K uses the same training distribution, optimizer settings, and $10^{-4}$ smoothed-loss trigger as VO. No Q/K seed reached that threshold, so all three entered cooldown at the $150{,}000$-update cap and finished at $195{,}000$ optimizer updates, more than any VO seed (Table~\ref{tab:baseline_seed_runs}).
The final Q/K smoothed losses ranged from $0.398$ to $0.402$. Figure~\ref{fig:synthetic_training_trajectory} compares the training curves with VO seed~$42$.

\begin{figure}[!htbp]
    \centering
    \includegraphics[width=0.76\textwidth]{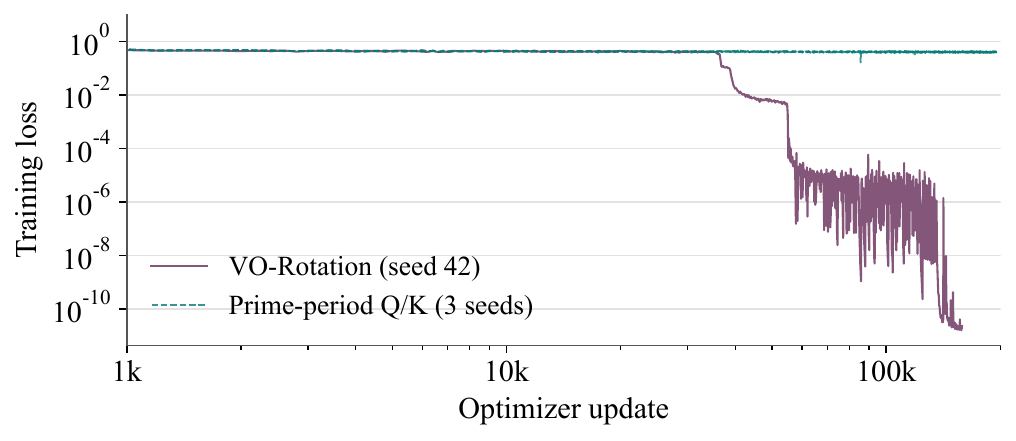}
    \caption{Training loss for VO-Rotation (seed $42$) and prime-period Q/K (all three seeds), with both axes logarithmic. The VO loss drops after a plateau; the three Q/K curves stay near $0.4$ through the end of training.}
    \label{fig:synthetic_training_trajectory}
\end{figure}

We evaluate Q/K on $96$ strings per length, condition, and seed. Under teacher forcing, we score each prediction given the correct prefix. For memory efficiency, we evaluate one string at a time and process up to $1024$ query positions at once without changing which keys each query can attend to.
The length-generalization plots use the common grid $8,16,\ldots,10{,}000$.
On that grid, Q/K exact-copy accuracy is zero for every seed and condition at $n=2000$ and at each evaluated length from $2008$ to $5000$.
An exact-copy score of zero does not tell us whether the model got the characters wrong or failed to stop at the right position. Table~\ref{tab:prime_qk_diagnostics} scores the payload and end marker separately at $n=2000$.
Q/K never gets the entire payload right on uniform-random or moderately repeated strings. On all-identical strings, it predicts every payload character correctly in every seed but never predicts EOS at the target stopping position.

\begin{table}[tbp]
    \centering
    \small
    \caption{Teacher-forced prime-period Q/K accuracy at source length $2000$, averaged across three seeds. Whole-payload accuracy requires all copied characters to be correct; EOS accuracy scores only the target stopping position.}
    \label{tab:prime_qk_diagnostics}
    \begin{tabular}{lccc}
        \toprule
        Condition & Payload characters & Whole payload & EOS \\
        \midrule
        Uniform-random & 0.0516 & 0.0000 & 0.0174 \\
        Moderately repeated & 0.9901 & 0.0000 & 0.0278 \\
        All-identical & 1.0000 & 1.0000 & 0.0000 \\
        \bottomrule
    \end{tabular}
\end{table}

\subsection{VO-Rotation Length Generalization up to 10{,}000 and the Period Product}
\label{subsec:vo_generalization_10000}

Figure~\ref{fig:synthetic_vo_per_seed} shows the individual seed curves for VO and its Q/K control.
All three VO seeds score $1.000$ at every evaluated $n\le2000$ in every condition.
We average each seed's accuracy over the $375$ extrapolation lengths $2008$--$5000$; across seeds, these averages have mean and sample standard deviation $0.9796\pm0.0189$ on uniform-random strings, $0.9963\pm0.0058$ on moderately repeated strings, and $0.9978\pm0.0020$ on all-identical strings.
Table~\ref{tab:vo_seed_generalization} gives each seed's mean and minimum.

At $5008$, the first evaluated length past the period product $5\times7\times11\times13=5005$, the mean accuracy across seeds falls to $0.0035$ on uniform-random and $0.2500$ on moderately repeated strings.
Averaged over the $625$ lengths $5008$--$10{,}000$ and the three seeds, accuracy on these two distributions is below $0.00001$ and $0.001$, respectively.
On all-identical strings, seeds $42$, $43$, and $44$ score $0.7292$, $0.9063$, and $0.0313$ at $5008$, and average $0.4196$, $0.1242$, and $0.0162$ over $5008$--$10{,}000$.

\begin{table}[!htbp]
    \centering
    \small
    \caption{VO exact-copy accuracy by training seed and distribution. Mean and minimum are over the $375$ evaluated lengths $2008$--$5000$; the final column gives accuracy at $5000$.}
    \label{tab:vo_seed_generalization}
    \begin{tabular}{clccc}
        \toprule
        Seed & Distribution & Mean & Minimum & $n=5000$ \\
        \midrule
        42 & Uniform-random & 0.9834 & 0.4375 & 0.9688 \\
        43 & Uniform-random & 0.9591 & 0.4688 & 0.8854 \\
        44 & Uniform-random & 0.9964 & 0.9583 & 1.0000 \\
        \midrule
        42 & Moderately repeated & 0.9993 & 0.9167 & 1.0000 \\
        43 & Moderately repeated & 0.9896 & 0.7188 & 0.9896 \\
        44 & Moderately repeated & 0.99997 & 0.9896 & 1.0000 \\
        \midrule
        42 & All-identical & 0.9962 & 0.9063 & 1.0000 \\
        43 & All-identical & 0.9973 & 0.8542 & 1.0000 \\
        44 & All-identical & 1.0000 & 1.0000 & 1.0000 \\
        \bottomrule
    \end{tabular}
\end{table}

\begin{figure}[!htbp]
    \centering
    \includegraphics[width=0.94\textwidth]{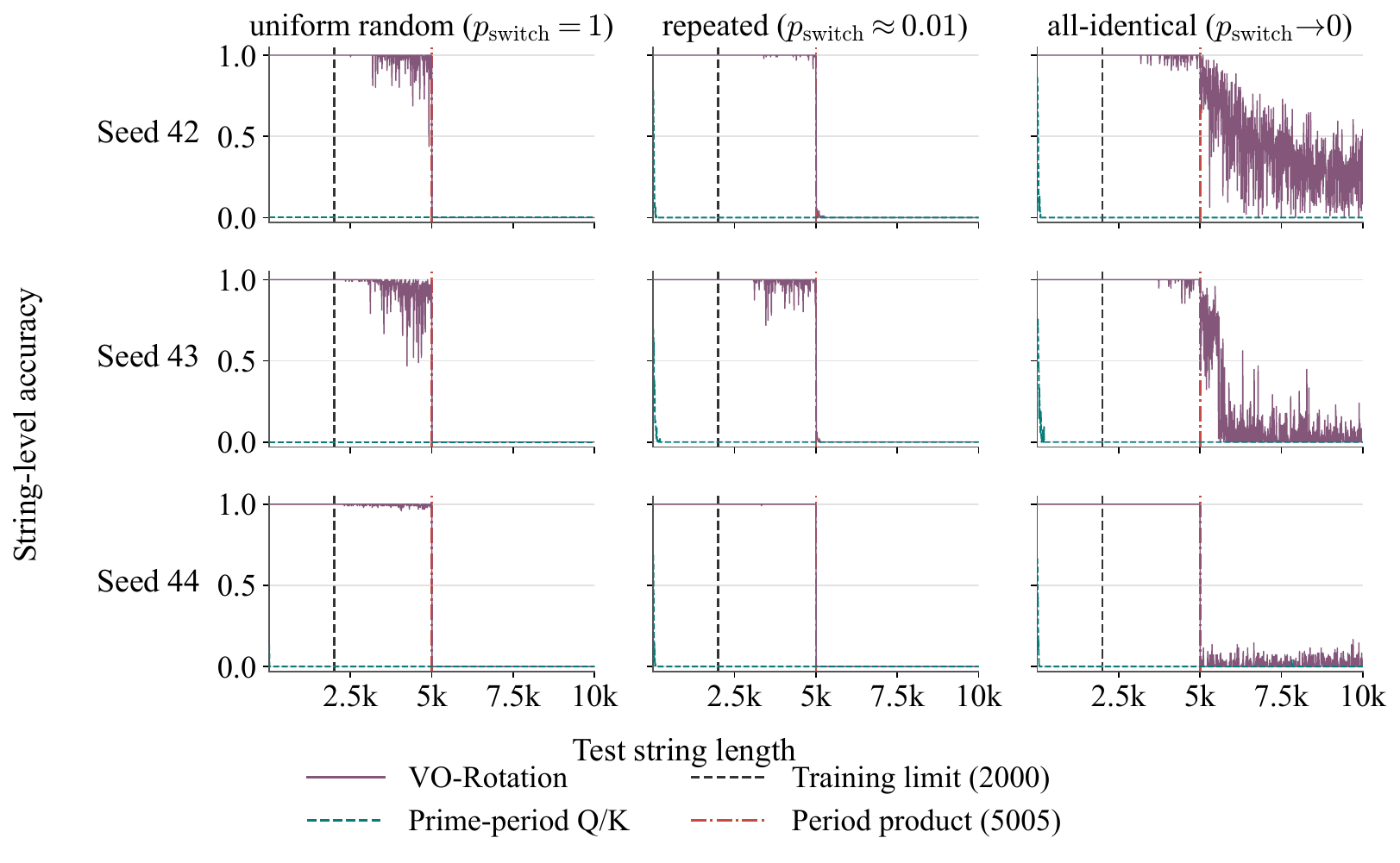}
    \caption{Teacher-forced exact-copy accuracy for VO-Rotation and prime-period Q/K by seed. Rows show seeds $42$, $43$, and $44$, and columns show the three evaluation distributions. The vertical lines mark the maximum training source length $2000$ and the period product $5005$.}
    \label{fig:synthetic_vo_per_seed}
\end{figure}
\clearpage

\section[Next-Token Copying at Fixed Context Length is in DLOGTIME-uniform AC0]{Next-Token Copying at Fixed Context Length is in DLOGTIME-uniform $\mathrm{AC}^0$}
\label{app:ac0}

At fixed context length~$n$, the position of \(\langle\Copy\rangle\) tells us how many tokens have already been copied, and thus which source-side position to read next. The circuit uses the delimiter position to select the one to copy among its input wires.
The circuit input is the whole current context, so $n$ here is the context length; in Section~\ref{sec:experiments}, $n$ denotes the source length instead.

Fix a finite source vocabulary~$\mathcal V$ and let $\Gamma=\mathcal V\cup\{\langle\BoS\rangle,\langle\EoI\rangle,\langle\Copy\rangle,\langle\EoS\rangle\}$.
Encode tokens by one-hot indicators $X_{i,\sigma}$, where $X_{i,\sigma}=1$ means that the token at position~$i$ is~$\sigma$.
For a fixed context length~$n$, consider reachable copy prefixes, i.e., contexts that arise during correct autoregressive copying (where every token generated after $\langle\Copy\rangle$ matches the corresponding source token):
\[
  \langle\BoS\rangle,\;s_1,\ldots,s_L,\;\langle\EoI\rangle,\;\langle\Copy\rangle,\;s_1,\ldots,s_k, \qquad 0\le k\le L,
\]
with total length $n=L+3+k$.
The \(\langle\Copy\rangle\) token sits at position $c=L+3$, so $k=n-c$, and the correct next token is read from source-side position $n-c+2=k+2$.
This position holds $s_{k+1}$ when $k<L$ and \(\langle\EoI\rangle\) when $k=L$.
The output map $\tau$ sends vocabulary tokens to themselves and sends \(\langle\EoI\rangle\) to \(\langle\EoS\rangle\).

For each candidate delimiter position~$c$, define $D_c=X_{c-1,\langle\EoI\rangle}\wedge X_{c,\langle\Copy\rangle}$.
The term gate $T_{c,\rho}$ tests both this delimiter boundary and the selected source token:
\[
  T_{c,\rho}:=D_c\wedge X_{n-c+2,\rho},
  \qquad \rho\in\mathcal V\cup\{\langle\EoI\rangle\}.
\]
For every output token $\sigma\in\Gamma$, define
\[
  Y_\sigma
  =
  \bigvee_{\substack{\lceil(n+3)/2\rceil\le c\le n,\\ \rho\in\mathcal V\cup\{\langle\EoI\rangle\}:\;\tau(\rho)=\sigma}} T_{c,\rho}.
\]
The delimiter range is equivalently $2\le n-c+2\le c-1$, which keeps the selected position on the source side.
% Figure~\ref{fig:ac0_wiring} shows this wiring at two consecutive context lengths.

% \input{src/fig_ac0_wiring}

\begin{proposition}[Autoregressive copying at fixed context length is in
DLOGTIME-uniform $\mathrm{AC}^0$]\label{prop:copy-ac0} The family of next-token maps on reachable copy prefixes, indexed by context length~$n$, is computable by a DLOGTIME-uniform $\mathrm{AC}^0$ circuit family of depth $O(1)$ and size $O(n|\Gamma|)$.
\end{proposition}

\begin{proof}
On this reachable-prefix domain, exactly one delimiter gate $D_c$ is active, at the true \(\langle\EoI\rangle,\langle\Copy\rangle\) boundary.
For that~$c$, $Y_\sigma$ reduces to $\bigvee_{\rho:\,\tau(\rho)=\sigma}X_{n-c+2,\rho}$, which equals~$1$ exactly when $\sigma$ is the image under~$\tau$ of the token at position $n-c+2$.
At the final copying step this position holds the source-side \(\langle\EoI\rangle\), so termination is handled by the same multiplexer and the fixed map \(\tau(\langle\EoI\rangle)=\langle\EoS\rangle\).

Each $D_c$ is a constant fan-in AND, each term $D_c\wedge X_{n-c+2,\rho}$ is a constant fan-in AND, and the final OR has fan-in $O(n)$.
Thus the circuit has constant depth and size $O(n|\Gamma|)$.
If tokens are represented by any fixed binary encoding rather than one-hot indicators, the predicates $X_{i,\sigma}$ are constant-size equality tests because $|\Gamma|$ is fixed, which changes only constant factors and constant depth.
DLOGTIME-uniformity requires that gate types and the direct-connection language be decidable in $O(\log n)$ time. The direct-connection language contains $(n,g,g')$ exactly when gate~$g$ feeds gate~$g'$ in the circuit for context length~$n$.
Gate indices specify one of four types: input gates $X_{i,\sigma}$, delimiter gates $D_c$, term gates $T_{c,\rho}$, and output gates $Y_\sigma$.
Given two gate indices, adjacency is determined by a constant number of comparisons and the affine relation $i=n-c+2$ on $O(\log n)$-bit indices, all decidable in $O(\log n)$ time.
Hence the family is DLOGTIME-uniform $\mathrm{AC}^0$.
\end{proof}

\begin{remark}[Fixed-context circuits and horizon-uniform models]
If the source length~$L$ and the number~$k$ of generated tokens are fixed as well, the map degenerates to a single wire from source position $k+2$.
The next-token map above instead selects among source positions using the delimiter-indexed table $c\mapsto n-c+2$ compiled into the wiring.
Horizon-uniform copying requires one transformer~$T_N$ with the same parameters for every source length $L\le N$ and every copying step.
Theorem~\ref{thm:impossibility} shows that finite-precision ECB transformers with saturation threshold $o(N)$ fail this requirement even on all-identical sources.
\end{remark}
\clearpage

\section{Proofs for Section~\ref{sec:theory}}
\label{app:theory_proofs}

We use the copying task of Section~\ref{sec:setup} and the anchor-relative offsets of Section~\ref{sec:construction}.

\paragraph{Proof dependencies.}

\begin{center}
\small
\begin{tabular}{@{}p{0.12\linewidth}p{0.80\linewidth}@{}}
\toprule
ECB & Repeated-addition eventual constancy (Fact~\ref{lem:stag}) $\to$ constant-block invariance (Lemma~\ref{lem:stag_prefix}) $\to$ terminal-run saturation (Lemma~\ref{lem:terminal_run_saturation}) $\to$ Theorem~\ref{thm:impossibility}; bias thresholds (Lemmas~\ref{lem:alibi_ecb} and~\ref{lem:hard_alibi_ecb}) $\to$ Corollary~\ref{cor:main}. \\
\midrule
VO with CRT periods & Layer-1 exact one-hot (Lemma~\ref{lem:anchor}) $\to$ fingerprint (Lemma~\ref{lem:fingerprint}) $\to$ Layer-2 logit bounds (Lemmas~\ref{lem:perhead}--\ref{lem:logit_bounds}) $\to$ Layer-2 exact one-hot (Lemma~\ref{lem:l2_exact}) $\to$
Theorem~\ref{thm:construction}. \\
\bottomrule
\end{tabular}
\end{center}

\paragraph{Notation.}
The table lists the symbols used in more than one subsection and where each is defined.
\begin{center}
\small
\begin{tabular}{@{}p{0.21\linewidth}p{0.48\linewidth}p{0.23\linewidth}@{}}
\toprule
Symbol & Meaning & Definition or convention \\
\midrule
\multicolumn{3}{@{}l}{\textit{Arithmetic and lower bound}} \\
$p_e$, $p_m$ & exponent and mantissa bits of the $p$-bit format & Section~\ref{app:fp_primitives} \\
$\varepsilon_{\mathrm{mach}}$ & machine epsilon $2^{-p_m}$ & Section~\ref{app:fp_primitives} \\
$U_p$ & logit magnitude at which $\exp(-U_p)$ rounds to zero & Definition~\ref{def:underflow} \\
$\bar n_\ell$ & largest eventual-constancy index in layer~$\ell$ & Definition~\ref{def:ecb_full} \\
$\kappa_\ell$ & terminal-run saturation threshold after layer~$\ell$ & Lemma~\ref{lem:terminal_run_saturation} \\
\midrule
\multicolumn{3}{@{}l}{\textit{VO construction}} \\
$a(t)$ & active anchor at position~$t$ & Section~\ref{sec:construction} \\
$\delta(t)$ & active-anchor offset $t-a(t)$ & Section~\ref{sec:construction} \\
$q_h$ & prime period of head~$h$ & Definition~\ref{def:qhead} \\
$\Pi$ & product of prime periods & Definition~\ref{def:vo_prime_full} \\
$q_{\max}$ & largest prime period & Definition~\ref{def:vo_prime_full} \\
$\mathbf u_q(r)$ & ideal unit vector for residue~$r$ modulo~$q$ & Section~\ref{app:vo_formal} \\
$\widehat R_q(r)$ & rounded period-$q$ rotation table & Section~\ref{app:vo_formal} \\
$\mathsf{mv}_p$ & matrix-vector product with $p$-bit rounding & Section~\ref{app:vo_formal} \\
$\mathbf f(t)$ & computed CRT fingerprint at position~$t$ & Section~\ref{app:vo_formal} \\
$C_{\mathrm{block}}$, $C_{\mathrm{sum}}$ & error constants for one prime block and for the summed logit & Equation~\eqref{eq:eff_gap} \\
$\varepsilon_{\mathrm{score}}$ & score-error bound in units of $A^2$ & Equation~\eqref{eq:eff_gap} \\
$\Delta$ & remaining mismatch penalty in units of $A^2$ & Equation~\eqref{eq:eff_gap} \\
$A$ & rotary-vector scale; logits are scaled by $A^2$ & Equation~\eqref{eq:eff_gap} \\
\bottomrule
\end{tabular}
\end{center}

\subsection{Finite-Precision Primitives}\label{app:fp_primitives}

Fix a $p$-bit floating-point format with $1$ sign bit, $p_e$ exponent bits, and $p_m$ mantissa bits ($1+p_e+p_m=p$).
Write $\mathsf{fl}(\cdot)$ for round-to-nearest-even.
Let $\calF_p$ be the finite set of representable floats, including signed zero.
There are at most $2^p$ possible $p$-bit encodings.
We write $\varepsilon_{\mathrm{mach}}:=2^{-p_m}$.
For lower-bound executions, we assume all arithmetic intermediates remain finite and representable.
In particular, the recursive ALiBi case uses a bias recurrence without overflow (Lemma~\ref{lem:alibi_ecb}).

\begin{definition}[Underflow threshold]\label{def:underflow}
When the set below is nonempty, define
\[
  U_p := \min\{m\in\calF_p:m>0\ \text{and}\ \mathsf{fl}(\exp(-m))=0\}.
\]
Thus $U_p$ is itself representable.
\end{definition}

\paragraph{Floating-point format for the VO construction.}
We use a binary format with gradual underflow and normal exponents from $2-2^{p_e-1}$ through $2^{p_e-1}-1$.
Its smallest positive subnormal is $2^{2-2^{p_e-1}-p_m}$.
Under round-to-nearest-even, $\exp(-m)$ rounds to zero exactly when it is at most half this subnormal ($2^{1-2^{p_e-1}-p_m}$), i.e., $m\ge(2^{p_e-1}+p_m-1)\ln2$.
The threshold $U_p$ is the least representable number at or above this value.
For the bit allocations in Definition~\ref{def:vo_prime_full}\textup{(B4)} below, $p_m=O(2^{p_e})$, so this number exists for all sufficiently large~$N$ and $U_p=\Theta(2^{p_e})$.
Rounding the real threshold upward to a representable number changes it by at most a constant factor; the exponent range also contains $U_p$, $-U_p$, and $-2U_p$.

\begin{fact}[Repeated-addition eventual constancy]\label{lem:stag}
Let $u$ be a $p$-bit float, $S_0$ an arbitrary $p$-bit float, and $S_{n+1}=\mathsf{fl}(S_n+u)$.
If every $S_n$ remains finite, then $S_{n+1}=S_n$ for all $n\ge n^*$ for some $n^*\le 2^p$.
\end{fact}

\begin{proof}
For $u=0$ the claim is immediate.
For $u>0$, round-to-nearest-even satisfies $\mathsf{fl}(x+u)\ge x$ for every finite representable~$x$, so $(S_n)$ is non-decreasing.
A non-decreasing sequence in the finite set~$\calF_p$ becomes constant after at most~$2^p$ distinct $p$-bit states.
The case $u<0$ follows by applying the positive-increment case to $(-S_n)$.
\end{proof}

\begin{corollary}[Componentwise eventual constancy]\label{cor:comp}
Let $\mathbf u\in\calF_p^m$ be fixed and $(\mathbf S_{n+1})_c=\mathsf{fl}((\mathbf S_n)_c+u_c)$.
Then $\mathbf S_n$ is eventually constant, with eventual-constancy index $n^*=\max_c n^*_c\le 2^p$.
\end{corollary}

\begin{proof}
Apply Fact~\ref{lem:stag} to each coordinate and take the maximum of the coordinatewise eventual-constancy indices.
\end{proof}

\paragraph{Attention execution model.}
Each head at layer~$\ell$ computes, for query position~$i$, logits $e_{i,j}$ for $j\le i$, maximum logit $e_{\max}$, shifted weights $\hat w_j=\mathsf{fl}(\exp(e_{i,j}-e_{\max}))$, and a left-to-right accumulation of the softmax denominator~$Z$ and numerator~$\mathbf N$:
\[
  Z_j=\mathsf{fl}(Z_{j-1}+\hat w_j),\qquad (\mathbf N_j)_c
  =
  \mathsf{fl}\bigl((\mathbf N_{j-1})_c+ \mathsf{fl}(\hat w_j(v_j)_c)\bigr).
\]
The head-local pre-output is $\widetilde{\mathbf o}_i=\mathbf N_i/Z_i$.
Write $\mathbf o_i^{(\ell)}:=W_O^{(\ell)}\mathrm{Concat}(\widetilde{\mathbf o}_i^{(\ell,1)},\ldots,\widetilde{\mathbf o}_i^{(\ell,H)})$ for the projected attention output.
The layer update is
\begin{equation}\label{eq:layer}
  h^{(\ell)}_i := h^{(\ell-1)}_i +\mathbf o_i^{(\ell)}
  +\mathrm{MLP}^{(\ell)}\bigl(h^{(\ell-1)}_i,\mathbf o_i^{(\ell)}\bigr).
\end{equation}
Here $\mathrm{MLP}^{(\ell)}$ is a fixed deterministic position-wise feed-forward map of the incoming state and local attention output; this notation includes an MLP applied to their sum.
All operations are deterministic in $p$-bit arithmetic.

\begin{lemma}[Constant-block invariance]\label{lem:stag_prefix}
Consider a left-to-right softmax accumulation that processes a contiguous block of at least~$2^p$ positions, each contributing the same shifted weight~$w$ and weighted value~$w\cdot\mathbf v$.
After the first $2^p$ positions in this block, the accumulator state is fixed.
Inserting or removing copies of the same addend, as long as the block keeps at least $2^p$ positions, does not change the state upon exit from the block.
\end{lemma}

\begin{proof}
Apply Fact~\ref{lem:stag} to the scalar denominator~$Z$.
For the numerator, apply Corollary~\ref{cor:comp} coordinatewise to the fixed vector addend $w\cdot \mathbf v$.
After at most~$2^p$ identical addends, both the denominator and every numerator coordinate have reached a fixed point for this repeated addition.
Any additional copies of the same addend therefore leave the accumulator unchanged.
\end{proof}

\subsection[ECB Lower Bound]{$\mathcal T_{\mathrm{ECB}}$ Lower Bound}\label{app:ecb_proof}

Throughout this subsection, fix the horizon~$N$ and write $p=p(N)$ for the precision at that horizon.
All contexts considered below lie in the model domain of~$T_N$.

\paragraph{Terminal-run contexts.}
For a fixed prefix $x\in\Gamma^*$, token $\sigma\in\mathcal V$, and $k\ge0$, consider the context $(x,\sigma^k)$.
We write $h^{(\ell)}_{\mathrm{last}}(x,\sigma^k)$ for its last-token hidden state.
By causality, $h^{(\ell)}_{\mathrm{last}}(x,\sigma^j)$ is also the state of the $j$-th terminal-run token inside $(x,\sigma^k)$ for any $k\ge j\ge1$.

\begin{definition}[Formal $\mathcal T_{\mathrm{ECB}}$ model]
\label{def:ecb_full}
Fix a horizon~$N$ and write $D$ for the depth of~$T_N$.
The class $\mathcal T_{\mathrm{ECB}}$ consists of models satisfying:
\begin{enumerate}
\item[\textup{(A0)}] Initial hidden states are position-independent token embeddings: $h_i^{(0)}=E(x_i)$, with no absolute positional embeddings, index features, counters, or other access to~$i$.
\item[\textup{(A1)}] The model is a depth-$D$ causal transformer-style model with arbitrary width and skip connections as in~\eqref{eq:layer}.
\item[\textup{(A2)}] Execution is deterministic $p$-bit execution as in Section~\ref{app:fp_primitives}, including the left-to-right softmax accumulation described above.
\item[\textup{(A3)}] At layer~$\ell$ and head~$h$, queries, keys, and values are functions only of layer-$(\ell-1)$ hidden states.
\item[\textup{(A4)}] Positional information enters only through a distance-dependent bias.
For visible keys, the numerical logit is
\[
  e^{(\ell,h)}_{i,j}
  =
  \mathsf{fl}\bigl((q^{(\ell,h)}_i)^\top k^{(\ell,h)}_j -b^{(\ell)}_h(i-j)\bigr).
\]
A hard-mask value $b_h^{(\ell)}(d)=+\infty$ means that the key is removed before numerical evaluation, with shifted weight exactly zero; it is not an infinite arithmetic intermediate.
Visible-key biases and all numerical arithmetic intermediates are finite.
\item[\textup{(A5)}] Each $b^{(\ell)}_h$ is eventually constant: there is $n^{(\ell)}_h$ such that $b^{(\ell)}_h(d+1)=b^{(\ell)}_h(d)$ for all $d\ge n^{(\ell)}_h$.
We write $\bar n_\ell:=\max_h n^{(\ell)}_h$.
\item[\textup{(A6)}] The readout is deterministic and position-free from the final hidden state.
\end{enumerate}
\end{definition}

\subsubsection{Terminal-run saturation and the copying lower bound}

Throughout the lower-bound proof, fix a prefix~$x$ and a token~$\sigma$.
We compare $(x,\sigma^k)$ and $(x,\sigma^{k+1})$: once the terminal run is sufficiently long, another copy of~$\sigma$ no longer changes the next-token logits.
The threshold at each layer allows $\bar n_\ell$ positions for the distance biases to become constant and $2^p$ identical contributions for the softmax accumulators to stop changing.
The following layerwise statement proves the terminal-run saturation lemma in the main text (Lemma~\ref{lem:saturated_suffix_self_loop}).

\begin{lemma}[Layerwise terminal-run saturation]\label{lem:terminal_run_saturation}
Fix $x\in\Gamma^*$ and $\sigma\in\mathcal V$.
Define
\[
  \kappa_0:=0,\qquad \kappa_\ell:=\kappa_{\ell-1}+\bar n_\ell+2^p .
\]
For every layer~$\ell$, the state $h^{(\ell)}_{\mathrm{last}}(x,\sigma^j)$ is independent of~$j$ for $j\ge\max\{1,\kappa_\ell\}$.
Consequently, for $k,k'\ge\kappa_D$,
\[
  h^{(D)}_{\mathrm{last}}(x,\sigma^k)
  =
  h^{(D)}_{\mathrm{last}}(x,\sigma^{k'})
\]
bit-for-bit.
Hence $T_N$ produces identical next-token logits on $(x,\sigma^k)$ and $(x,\sigma^{k'})$.
\end{lemma}

\begin{proof}
We induct on~$\ell$.
The case $\ell=0$ follows from position-independent token embeddings, since every terminal-run token is~$\sigma$.
Assume the claim at layer~$\ell-1$, and consider the query at index~$j$ within the terminal run, with
\[
  j\ge\kappa_\ell=\kappa_{\ell-1}+\bar n_\ell+2^p .
\]
By the induction hypothesis and causality, all terminal-run tokens with index at least $\max\{1,\kappa_{\ell-1}\}$ (including~$j$ itself) share one previous-layer state, hence the same query, key, and value.
For any head~$h$, we want to isolate the keys whose previous-layer states and distance biases have both become constant. The first condition holds from run index $\max\{1,\kappa_{\ell-1}\}$ onward; the second holds through $j-n_h^{(\ell)}$. This gives three blocks:
\[
  \text{prefix and early run tokens}
  \quad|\quad
  \max\{1,\kappa_{\ell-1}\},\ldots,j-n_h^{(\ell)}
  \quad|\quad
  j-n_h^{(\ell)}+1,\ldots,j .
\]
The last block is empty if $n_h^{(\ell)}=0$.
The first block has fixed hidden states by causality and lies beyond the bias threshold. The last block moves with the query, but its states are identical and its distances from the query are always $n_h^{(\ell)}-1,\ldots,0$. Thus neither block's sequence of visible logits and values depends on~$j$.
Only the middle block changes length. It contains at least $2^p$ positions, all with the same state and the same bias. If visible, they contribute copies of one logit, so changing their number does not change the maximum; if masked, they do not enter the maximum at all.

With the maximum fixed, the middle block adds the same shifted weight and weighted value at each position (zero if masked). By constant-block invariance (Lemma~\ref{lem:stag_prefix}), the accumulator leaves this block in the same state regardless of its length. The last block then updates that state identically.
Thus every head output, and hence the layer-$\ell$ hidden state after the deterministic layer update, is independent of~$j$.
This proves the induction.
The final hidden-state equality follows by taking $\ell=D$, and the logit equality follows from the position-free deterministic readout~\textup{(A6)}.
\end{proof}

\begin{proof}[Proof of Theorem~\ref{thm:impossibility}]
Since $\kappa_D=o(N)$, for all sufficiently large~$N$ and every $L\in[N/3,N/2]$ we have $L-1\ge\kappa_D$.
Let
\[
  x_L:=\langle\BoS\rangle,\sigma^L,\langle\EoI\rangle,\langle\Copy\rangle .
\]
The last copying step and the stopping step for this same prompt require different outputs: $\sigma$ on $(x_L,\sigma^{L-1})$ and $\langle\EoS\rangle$ on $(x_L,\sigma^L)$.
If the first $L$ generated tokens are not all~$\sigma$, the model has already failed. Otherwise, greedy decoding outputs~$\sigma$ on $(x_L,\sigma^{L-1})$.
By Lemma~\ref{lem:terminal_run_saturation}, this context has the same logits as every $(x_L,\sigma^k)$ with $k\ge L-1$ in the model domain. Greedy decoding therefore outputs~$\sigma$ at the stopping step and continues to do so for as long as the context remains in the model domain.
\end{proof}

\begin{definition}[ALiBi and Hard-ALiBi distance biases]\label{def:alibi_hard_alibi}
For a query position~$i$ and key position~$j\le i$, write $d=i-j$.
In our sign convention, the attention logit is $q_i^\top k_j-b_h(d)$.

Following \citet{press2022alibi}, an \emph{ALiBi} head with slope $m_h\ge0$ uses the distance penalty
\[
  b_h^{\mathrm{ALiBi}}(d)=m_h d .
\]
This exact-real bias is not eventually constant when $m_h>0$; the same is true if one computes $m_hd$ directly and rounds each value independently.
The recursive implementation of ALiBi mentioned in Corollary~\ref{cor:main} is computed by the rounded recurrence
\[
  b_h(0)=0,\qquad b_h(d+1)=\mathsf{fl}(b_h(d)+m_h).
\]

Following \citet{jelassi2024repeat}, an ideal \emph{Hard-ALiBi} head with window length $w_h\in\mathbb N$ masks all positions outside the window:
\[
  b_h^{\mathrm{H\text{-}ALiBi}}(d)
  =
  \begin{cases}
  0, & 0\le d<w_h,\\
  +\infty, & d\ge w_h.
  \end{cases}
\]
Thus the head has an unpenalized causal window of length~$w_h$.
In finite-precision execution, hard-masked keys are removed before the softmax and receive shifted weight zero.
Replacing the mask by a representable constant penalty $B_h$ also gives an eventually constant bias (Lemma~\ref{lem:hard_alibi_ecb}), so both versions satisfy the bias condition.
Unmasked Hard-ALiBi heads are NoPE heads, $b_h\equiv 0$.
\end{definition}

\paragraph{Bias thresholds for the corollary.}
To apply the lower bound to the three cases in Corollary~\ref{cor:main}, it remains to bound their eventual-constancy indices.

\begin{lemma}[Recursive implementation of ALiBi is eventually constant]
\label{lem:alibi_ecb}
Let $m_h\ge0$ be a $p$-bit float and define $b_h(0)=0$, $b_h(d+1)=\mathsf{fl}(b_h(d)+m_h)$.
Assuming finite intermediates do not overflow, there exists $n_h\le 2^p$ such that $b_h(d+1)=b_h(d)$ for all $d\ge n_h$.
\end{lemma}

\begin{proof}
This is Fact~\ref{lem:stag} with $S_0=0$ and $u=m_h$; the no-overflow assumption keeps every $b_h(d)$ finite.
\end{proof}

\begin{lemma}[Hard-ALiBi is eventually constant]\label{lem:hard_alibi_ecb}
A finite-window Hard-ALiBi head with window length $w_h$ as in Definition~\ref{def:alibi_hard_alibi} is constant outside the window: either the masked contribution is zero in the ideal hard-mask implementation, or $b_h(d+1)=b_h(d)$ for all $d\ge w_h$ in the finite-penalty implementation.
Thus its eventual-constancy index is at most~$w_h$.
\end{lemma}

\begin{proof}
For every $d\ge w_h$, both $d$ and $d+1$ fall outside the unpenalized window.
In the ideal hard-mask version, both positions are removed from the softmax, equivalently both shifted weights are fixed at~$0$.
In a finite-penalty version, both distances receive the same outside-window penalty~$B_h$, so $b_h(d)=b_h(d+1)=B_h$.
\end{proof}

\paragraph{General Hard-ALiBi window condition.}
For a Hard-ALiBi family, Theorem~\ref{thm:impossibility} applies whenever
\[
  \sum_{\ell=1}^D\Bigl(\max_h w_h^{(\ell)}+2^p\Bigr)=o(N),
\]
where the maximum is over the hard-masked heads in layer~$\ell$ and is $0$ if the layer has none.
The main-text corollary uses the simpler sufficient condition that every window is at most~$2^p$.

\begin{proof}[Proof of Corollary~\ref{cor:main}]
We bound $\kappa_D$ by $2D\,2^p$ in each of the three cases.
For NoPE, $b_h^{(\ell)}\equiv0$ for every layer and head, so $\bar n_\ell=0$ and $\kappa_D=D\,2^p$.
For the recursive implementation of ALiBi with finite intermediates, Lemma~\ref{lem:alibi_ecb} gives $n_h^{(\ell)}\le 2^p$ for every head, so $\bar n_\ell\le 2^p$ and $\kappa_D\le2D\,2^p$.
For bounded-window Hard-ALiBi with all windows at most~$2^p$, Lemma~\ref{lem:hard_alibi_ecb} gives $n_h^{(\ell)}\le2^p$ for every hard-masked head; unmasked heads have $n_h^{(\ell)}=0$.
Hence again $\bar n_\ell\le2^p$ and $\kappa_D\le2D\,2^p$.
Since $D(N)2^{p(N)}=o(N)$,
\[
  \kappa_D \le 2D(N)\,2^{p(N)} = o(N).
\]
Theorem~\ref{thm:impossibility} applies in all cases.
In particular, $D=N^{0.01}$ and $p\le0.9\log_2N$ give $D\,2^p\le N^{0.91}=o(N)$.
Even at $p=\log_2N-C\log_2\log_2N$, the condition holds for every depth $D=o((\log_2N)^C)$.
\end{proof}

\subsection{Formal VO Construction with CRT Periods}
\label{app:vo_formal}

This subsection formalizes the VO-rotated head of Definition~\ref{def:vo_rotated_head} and the class $\mathcal T_{\mathrm{VO}}$ of Definition~\ref{def:vo}.
Recall that $a(t)$ is the \(\langle\BoS\rangle\) position until \(\langle\Copy\rangle\) appears, and the \(\langle\Copy\rangle\) position thereafter; $\delta(t)=t-a(t)$ is the offset from that anchor.
The source tokens have offsets $1,\ldots,L$, the source-side \(\langle\EoI\rangle\) has offset~$L+1$, and the copy region starts at offset~$0$.
Throughout this subsection, $R_q$ denotes the fixed two-dimensional period-$q$ rotation
\[
  R_q :=
  \begin{pmatrix}
  \cos(2\pi/q) & -\sin(2\pi/q)\\
  \sin(2\pi/q) & \cos(2\pi/q)
  \end{pmatrix},
\]
with powers interpreted modulo~$q$.
The ideal unit vector for residue~$r$ is $\mathbf u_q(r):=R_q^r\mathbf e_1=(\cos(2\pi r/q),\sin(2\pi r/q))^\top$.
We store rounded rotations as $\widehat R_q(r):=\mathsf{fl}(R_q^r)$, indexed by the exact residue $r\bmod q$ and rounded entrywise.
Let $\mathsf{mv}_p(B,\mathbf v)$ denote a matrix-vector product with each multiplication and addition rounded to $p$ bits.
Layer~1 stores one computed rotary vector $\mathbf f_h(t)$ per prime period.
Together they form the fingerprint $\mathbf f(t):=(\mathbf f_h(t))_{h=1}^H$.
The ideal counterpart of $\mathbf f_h(t)$ is $\mathbf u_{q_h}(\delta(t))$; Lemma~\ref{lem:fingerprint} bounds the error between them.

\begin{definition}[Prime-$q_h$ VO head]\label{def:qhead}
A \emph{prime-$q_h$ VO head} is a NoPE causal attention head with one-dimensional query/key coordinates, ordinary logits $e_{i,j}=\mathsf{fl}(q_i k_j)$, a two-dimensional rotary value subspace, and prime period~$q_h$.
Before rounding, its rotary head-local pre-output is
\[
  \widetilde{\mathbf o}_t^{\mathrm{rot}}
  =
  R_{q_h}^{t}\sum_{j\le t}\alpha_{t,j}R_{q_h}^{-j}\mathbf v_j^{\mathrm{rot}}.
\]
If this head attends one-hot to position~$j$ and $\mathbf v_j^{\mathrm{rot}}=\mathbf e_1$, its ideal rotary output is $\mathbf u_{q_h}(t-j)$; the computed output uses the rounded rotation tables.
\end{definition}

\begin{definition}[Formal $\mathcal T_{\mathrm{VO}}$ instantiation with CRT periods]
\label{def:vo_prime_full}
For the construction theorem, $\mathcal T_{\mathrm{VO}}$ is instantiated with prime periods as follows.
\begin{enumerate}
\item[\textup{(B1)}] Layer~1 has $H$ prime-$q_h$ VO heads, one for each of the first $H$ primes; each has one two-dimensional rotary block.
The output projection writes their $2H$ head-local pre-output coordinates into a fingerprint subspace and leaves the token-embedding subspace untouched; $\mathrm{MLP}^{(1)}=0$.
\item[\textup{(B2)}] Layer~2 is a single readout attention head whose query and key vectors are partitioned into $H$ prime blocks.
The index $h$ in $\mathbf q^{(2,h)}_t$ and $\mathbf k^{(2,h)}_j$ denotes a block within this head.
Its logit is the sequential $p$-bit sum
\begin{equation}\label{eq:summed_logit}
  e^{(2)}_{t,j}
  =
  \flsum_{h=1}^{H} \langle \mathbf q^{(2,h)}_t,\mathbf k^{(2,h)}_j\rangle_p .
\end{equation}
Here $\flsum$ denotes left-to-right rounded summation, and $\langle\cdot,\cdot\rangle_p$ denotes an inner product evaluated in $p$-bit arithmetic.
Its value vector is the preserved token embedding at position~$j$, and its output projection writes only to the token-readout block read by the final linear readout in~\textup{(B6)}; $\mathrm{MLP}^{(2)}=0$.
\item[\textup{(B3)}] $H=\Theta(\log N/\log\log N)$, so the largest prime period is $q_{\max}:=\max_h q_h=\Theta(\log N)$ and the product of the periods is $\Pi:=\prod_h q_h>N+1$.
\item[\textup{(B4)}] $p_m=c_m\lceil\log_2\log_2\Pi\rceil$ with $c_m\ge4$ and $p_e=c_e\lceil\log_2\log_2\Pi\rceil$ with $c_e$ a sufficiently large absolute constant.
\item[\textup{(B5)}] Each position has exact access to its index modulo~$q_h$, which selects the corresponding rounded rotation-table entry.
Its value rotations apply $\widehat R_{q_h}(r)$ using $\mathsf{mv}_p$.
Implementation and storage costs are given in Section~\ref{app:vo_implementation}.
\item[\textup{(B6)}] For any fixed finite vocabulary~$\mathcal V$, choose $d_{\mathrm{tok}}=d_{\mathrm{tok}}(|\mathcal V|)$ large enough for disjoint token-embedding and token-readout blocks.
Token embeddings for $\mathcal V\cup\{\langle\BoS\rangle,\langle\EoI\rangle,\langle\Copy\rangle,\langle\EoS\rangle\}$ occupy the first block; the fingerprint and token-readout blocks of every input embedding are zero.
The final linear readout reads only the token-readout block, and its weights make greedy argmax return each vocabulary token from its own returned vector and \(\langle\EoS\rangle\) from the returned \(\langle\EoI\rangle\) vector.
One may take one-hot token embeddings with identity rows for $\mathcal V$, an \(\langle\EoS\rangle\) row that scores the \(\langle\EoI\rangle\) coordinate, and zero readout weights on the other coordinates.
\end{enumerate}
\end{definition}

\paragraph{Layer-1 instantiation.}
For each prime head~$h$,
\begin{equation}\label{eq:qk1}
  (q_t^{(1)},k_t^{(1)})
  =
  \begin{cases}
    (0,-U_p), & x_t=\langle\BoS\rangle,\\
    (1,0), & x_t=\langle\Copy\rangle,\\
    (1,-2U_p), & \text{otherwise},
  \end{cases}
\end{equation}
where $x_t$ is the token at position~$t$, with value vector $\mathbf v_t^{(1)}=\mathbf e_1$ in the two-dimensional rotary subspace.
For a non-\(\langle\BoS\rangle\) query, these logits favor the active anchor: \(\langle\Copy\rangle\) if visible, and \(\langle\BoS\rangle\) otherwise. The gap to every other visible key is at least~$U_p$.
The resulting computed component $\mathbf f_h(t)$ approximates $\mathbf u_{q_h}(\delta(t))$ through two rounded rotations, as shown in Lemma~\ref{lem:fingerprint}.

\paragraph{Layer-2 instantiation.}
At a copy-region query~$t$, the next source-side position has offset $\delta(t)+1$.
Layer~2 rotates the query fingerprint by one step and compares it with each stored key fingerprint.
For each prime block~$h$, choose
\begin{equation}\label{eq:qk2}
  \mathbf q_t^{(2,h)}
  =
  \bigl[A\,\mathsf{mv}_p(\widehat R_{q_h}(1),\mathbf f_h(t)),\;-A^2\bigr],
  \qquad
  \mathbf k_j^{(2,h)}
  =
  \bigl[A\,\mathbf f_h(j),\;1\bigr],
\end{equation}
where $\widehat R_{q_h}(1)$ is the rounded one-step rotation and $A$ is a power-of-two scale chosen below.
With ideal fingerprints and exact arithmetic, the centered inner product is
\[
  A^2\left(\left\langle \mathbf u_{q_h}(\delta(t)+1),\mathbf u_{q_h}(\delta(j))\right\rangle-1\right).
\]
After subtracting $1$, the ideal block score is zero when the residues match and negative when they differ. The target at offset $\delta(t)+1$ matches in every block. Every wrong key has at least one negative term, and the other blocks cannot cancel it.
We choose the scale~$A^2$ and the precision so that, after subtracting the maximum logit, all non-target exponentials round to zero.
Let $C_{\mathrm{block}}$ be a universal error constant for one prime block, including the stored fingerprints, the one-step query rotation, and the rounded inner product.
Let $C_{\mathrm{add}}$ be the constant of Lemma~\ref{lem:l2_sum_rounding} with $C_z=3$, and set $C_{\mathrm{sum}}:=C_{\mathrm{block}}+C_{\mathrm{add}}$.
These constants depend only on the round-to-nearest error model and the constant block dimension, not on~$N$, $H$, $q_h$, $p$, or the input length.
Define the score-error budget $\varepsilon_{\mathrm{score}}$, the remaining mismatch penalty~$\Delta$, and the scale~$A$ by
\begin{equation}\label{eq:eff_gap}
  \varepsilon_{\mathrm{score}}:=C_{\mathrm{sum}} H^2\varepsilon_{\mathrm{mach}},\qquad
  \Delta:=8/q_{\max}^2-\varepsilon_{\mathrm{score}},\qquad
  A:=2^{\left\lceil \frac12\log_2(4U_p/\Delta) \right\rceil}.
\end{equation}
The target logit will lie within $A^2\varepsilon_{\mathrm{score}}$ of zero, while every wrong logit will be at most $-A^2\Delta$.
Their separation is therefore at least $A^2(\Delta-\varepsilon_{\mathrm{score}})$, which must exceed the underflow threshold~$U_p$.
Under~\textup{(B4)}, $\varepsilon_{\mathrm{score}}=o(1/q_{\max}^2)$, as computed in the proof of Theorem~\ref{thm:construction}.
Since $A$ is a power of two, both $A$ and $A^2$ are exactly representable binary floating-point numbers whenever the exponent range contains their exponents.
Throughout the Layer-2 proof, take $N$ large enough that $\varepsilon_{\mathrm{score}}\le2/q_{\max}^2$.
Then $\Delta\ge6/q_{\max}^2>0$ and $\varepsilon_{\mathrm{score}}\le\Delta/3$.

\subsection{Proofs for the VO Construction with CRT Periods}

We first verify anchor selection, then bound the error in the stored fingerprints and the Layer-2 logits.

\subsubsection{Layer-1 anchor selection and fingerprints}

\begin{lemma}[Layer-1 exact one-hot]\label{lem:anchor}
Under~\eqref{eq:qk1}, Layer~1 satisfies $\alpha^{(1)}_{t,a(t)}=1$ bit-for-bit for every non-\(\langle\BoS\rangle\) position.
\end{lemma}

\begin{proof}
For every non-\(\langle\BoS\rangle\) query position, $q_t^{(1)}=1$, so the ordinary attention logit to candidate~$j$ is $k_j^{(1)}$.

If \(\langle\Copy\rangle\) has not appeared in the causal prefix up to~$t$, the visible \(\langle\BoS\rangle\) key is $-U_p$, while every other visible key is $-2U_p$.
Hence the \(\langle\BoS\rangle\) position \(a(t)\) is the unique maximizer.
After max-shift, the \(\langle\BoS\rangle\) weight is $\fl(\exp(0))=1$, and every other shifted logit is at most $-U_p$, so its exponential rounds to~$0$ by Definition~\ref{def:underflow}.

If \(\langle\Copy\rangle\) has appeared, the unique \(\langle\Copy\rangle\) position has key~$0$, \(\langle\BoS\rangle\) has key $-U_p$, and all other visible positions have key $-2U_p$.
Thus the \(\langle\Copy\rangle\) position \(a(t)\) is the unique maximizer.
Again the winner contributes weight~$1$ and every non-winner is at least $U_p$ below the maximum after max-shift, so all non-winner weights round to~$0$.
In both cases the denominator is exactly~$1$ and the attention is bit-for-bit one-hot on~$a(t)$.
\end{proof}

With this one-hot attention, the denominator is $1$ and all non-anchor numerator contributions are zero.
Pre-rotation of $\mathbf e_1$ selects the first column of the rounded table exactly, and accumulation preserves that column.
The computed rotary output is therefore
\begin{equation}\label{eq:fp_output}
  \mathbf f_h(t)
  =
  \mathsf{mv}_p\!\left(\widehat R_{q_h}(t),\widehat R_{q_h}(-a(t))\mathbf e_1\right).
\end{equation}
The same expression holds at $\langle\BoS\rangle$, whose causal attention contains only itself.

Both the fingerprint separation below and the Layer-2 margin rest on a lower bound for the gap between distinct residues on a period-$q$ circle.

\begin{fact}[Cosine gap]\label{fact:cos}
For integer $q\ge2$ and integers $r\not\equiv r'\pmod q$,
\[
  1-\cos\bigl(2\pi(r-r')/q\bigr)\ \ge\ 1-\cos(2\pi/q)\ \ge\ 8/q^2.
\]
\end{fact}

\begin{proof}
The residue of $r-r'$ modulo~$q$ lies in $\{1,\ldots,q-1\}$, and over this range $\cos(2\pi m/q)$ is largest at $m=1$ and $m=q-1$, where it equals $\cos(2\pi/q)$.
This gives the first inequality.
For $\theta\in[0,\pi]$, $1-\cos\theta=2\sin^2(\theta/2)$, and $\sin x\ge 2x/\pi$ on $[0,\pi/2]$.
Thus $1-\cos\theta\ge 2\theta^2/\pi^2$.
Substituting $\theta=2\pi/q$ gives $1-\cos(2\pi/q)\ge 8/q^2$.
\end{proof}

\begin{lemma}[Fingerprint]\label{lem:fingerprint}
For all sufficiently large~$N$, the computed Layer-1 outputs satisfy:
\begin{enumerate}
\item[\textup{(i)}] $\mathbf f_h(t)$ depends only on the pair $(t\bmod q_h,a(t)\bmod q_h)$.
\item[\textup{(ii)}] $\|\mathbf f_h(t)-\mathbf u_{q_h}(\delta(t))\|_\infty \le 8\varepsilon_{\mathrm{mach}}$.
\item[\textup{(iii)}] $\mathbf f(t)$ uniquely identifies $\delta(t)$ for $\delta(t)\le N+1<\Pi$.
\end{enumerate}
\end{lemma}

\begin{proof}
Part~\textup{(i)} follows from~\eqref{eq:fp_output} and exact table indexing.
For part~\textup{(ii)}, fix a head, abbreviate $q=q_h$ and $a=a(t)$, and put $v=R_q^{-a}\mathbf e_1$ and $\widehat v=\widehat R_q(-a)\mathbf e_1$.
Every rounded table entry has magnitude at most~$1$ and absolute error at most $\varepsilon_{\mathrm{mach}}$.
Thus $\|\widehat v-v\|_\infty\le\varepsilon_{\mathrm{mach}}$ and $\|\widehat v\|_\infty\le1$.
Using the matrix row-sum norm and $\|R_q^t\|_\infty\le\sqrt2$, we obtain
\[
  \|\widehat R_q(t)\widehat v-R_q^t v\|_\infty
  \le (2+\sqrt2)\varepsilon_{\mathrm{mach}}.
\]
Each coordinate of $\mathsf{mv}_p(\widehat R_q(t),\widehat v)$ requires two products of magnitude at most~$1$ and one sum of magnitude at most~$2$.
For sufficiently large~$N$, rounding these operations introduces at most $\varepsilon_{\mathrm{mach}}$ per product and $2\varepsilon_{\mathrm{mach}}$ for the sum, including underflow; the exponent range in~\textup{(B4)} contains these bounded intermediates without overflow.
The arithmetic error is therefore at most $4\varepsilon_{\mathrm{mach}}$ per coordinate.
Since $R_q^t v=\mathbf u_q(t-a)$ in exact arithmetic, the total error is at most $(6+\sqrt2)\varepsilon_{\mathrm{mach}}<8\varepsilon_{\mathrm{mach}}$, independently of $t$, $a$, and~$q$.

For part~\textup{(iii)}, two distinct offsets $\delta,\delta'\in\{0,\ldots,N+1\}$ differ modulo some prime~$q_h$ by CRT, since $|\delta-\delta'|<\Pi$.
For their distinct residues $r,r'$ in this head,
\[
  \|\mathbf u_{q_h}(r)-\mathbf u_{q_h}(r')\|_2^2 =2\bigl(1-\cos(2\pi(r-r')/q_h)\bigr) \ge 16/q_h^2
\]
by Fact~\ref{fact:cos}.
Hence the $\ell_\infty$ separation is at least $4/(q_h\sqrt2)>2/q_h$.
The two computation errors subtract at most $16\varepsilon_{\mathrm{mach}}$, which is smaller for sufficiently large~$N$, since $c_m\ge4$ and $q_h\le q_{\max}=\Theta(\log N)$.
Thus computed fingerprints belonging to distinct offsets remain distinct, even when their absolute positions and anchors differ.
\end{proof}

\subsubsection{Layer-2 matching and finite precision}

The goal is to make the target's softmax weight exactly one. CRT guarantees at least one mismatched prime block for each wrong key. We first bound how much rounding can erode that block's cosine penalty, then check that it survives summing the rounded scores. The scale~$A^2$ makes the remaining gap large enough that the non-target exponentials round to zero after max-shifting.

\begin{lemma}[Per-prime-block logit bounds]\label{lem:perhead}
Let $r_{\mathrm{query}}=(\delta(t)+1)\bmod q_h$ be the residue sought by the query and $r_{\mathrm{key}}=\delta(j)\bmod q_h$ the residue stored at the candidate key.
For all sufficiently large~$N$, we have
\[
  \left|
    \langle\mathbf q_t^{(2,h)},\mathbf k_j^{(2,h)}\rangle_p
    -A^2\bigl(\cos(2\pi(r_{\mathrm{query}}-r_{\mathrm{key}})/q_h)-1\bigr)
  \right|
  \le C_{\mathrm{block}} A^2\varepsilon_{\mathrm{mach}}.
\]
In particular, matching residues give an absolute value at most $C_{\mathrm{block}}A^2\varepsilon_{\mathrm{mach}}$, while mismatched residues give
\[
  \langle\mathbf q_t^{(2,h)},\mathbf k_j^{(2,h)}\rangle_p \le -A^2\bigl(8/q_h^2-C_{\mathrm{block}}\varepsilon_{\mathrm{mach}}\bigr).
\]
\end{lemma}

\begin{proof}
Write $\mathbf q_{\mathrm{rot}}:=\mathsf{mv}_p(\widehat R_{q_h}(1),\mathbf f_h(t))$ and $\mathbf k_{\mathrm{rot}}:=\mathbf f_h(j)$ for the unscaled rotary query and key components.
Lemma~\ref{lem:fingerprint} gives a uniform $8\varepsilon_{\mathrm{mach}}$ error for each stored component.
The additional one-step query rotation has bounded entries and a fixed number of rounded operations, so it adds only $O(\varepsilon_{\mathrm{mach}})$ error.
Consequently,
\[
  \mathbf q_{\mathrm{rot}}=\mathbf u_{q_h}(r_{\mathrm{query}})+O(\varepsilon_{\mathrm{mach}}), \qquad \mathbf k_{\mathrm{rot}}=\mathbf u_{q_h}(r_{\mathrm{key}})+O(\varepsilon_{\mathrm{mach}}).
\]
The exact inner product of the unrounded unit vectors is
\[
  \langle \mathbf u_{q_h}(r_{\mathrm{query}}),\mathbf u_{q_h}(r_{\mathrm{key}})\rangle =\cos(2\pi(r_{\mathrm{query}}-r_{\mathrm{key}})/q_h).
\]
Scaling by the power of two~$A$, multiplying coordinates, summing the two coordinates, and adding the centering term $-A^2$ add further $O(A^2\varepsilon_{\mathrm{mach}})$ error; $C_{\mathrm{block}}$ is the constant that covers these errors together with those above.
If $r_{\mathrm{query}}=r_{\mathrm{key}}$, the exact centered value is~$0$.
If $r_{\mathrm{query}}\neq r_{\mathrm{key}}$, Fact~\ref{fact:cos} gives $\cos(2\pi(r_{\mathrm{query}}-r_{\mathrm{key}})/q_h)-1\le -8/q_h^2$, which gives the stated upper bound after rounding.
\end{proof}

To combine the per-block bounds, we must also control rounding in the sum over prime blocks.

\begin{lemma}[Left-to-right summation error]\label{lem:l2_sum_rounding}
Let $z_1,\ldots,z_H$ be finite $p$-bit numbers with $|z_h|\le C_zA^2$.
Assume no overflow occurs during the deterministic left-to-right $p$-bit summation and $H\varepsilon_{\mathrm{mach}}$ is sufficiently small.
Then, for a constant $C_{\mathrm{add}}$ depending only on~$C_z$,
\[
  \left|
  \flsum_{h=1}^{H} z_h-\sum_{h=1}^{H}z_h
  \right|
  \le C_{\mathrm{add}} H^2A^2\varepsilon_{\mathrm{mach}} .
\]
\end{lemma}

\begin{proof}
Let $\widehat S_m$ be the rounded left-to-right partial sum and $S_m=\sum_{h=1}^m z_h$ the exact real partial sum.
Inductively, as long as $H\varepsilon_{\mathrm{mach}}$ is small, $|\widehat S_m|=O(HA^2)$, with constants depending only on~$C_z$.
Thus the $m$-th rounded addition has absolute rounding error $O(HA^2\varepsilon_{\mathrm{mach}})$.
Summing this bound over the $H$ additions gives the claim.
\end{proof}

At a copy-region query of offset~$k$, all visible copy-region offsets are at most~$k$.
The source side contains exactly one position of offset~$k+1$ for $0\le k\le L$, so the target is unique among visible positions.
The next lemma compares its logit with those of all other visible positions.

\begin{lemma}[Layer-2 logit bounds]\label{lem:logit_bounds}
Let copy position~$t$ have offset~$k$, and let $j^\star$ be the unique visible position with $\delta(j^\star)=k+1$.
For all sufficiently large~$N$, provided $k+1\le N+1<\Pi$:
\begin{enumerate}
\item[\textup{(i)}] $|e^{(2)}_{t,j^\star}|\le A^2\varepsilon_{\mathrm{score}}$.
\item[\textup{(ii)}] $e^{(2)}_{t,j}\le -A^2\Delta$ for every $j\neq j^\star$ with $j\le t$.
\end{enumerate}
\end{lemma}

\begin{proof}
For sufficiently large~$N$, Lemma~\ref{lem:perhead} gives $|\langle\mathbf q_t^{(2,h)},\mathbf k_j^{(2,h)}\rangle_p|\le 3A^2$ for every block and key, so Lemma~\ref{lem:l2_sum_rounding} applies with $C_z=3$.

For the correct match~$j^\star$, every prime block has $\delta(j^\star)\equiv k+1\pmod{q_h}$, so the residues match at every prime.
Lemma~\ref{lem:perhead} bounds the exact sum of the $H$ rounded block contributions by $HC_{\mathrm{block}}A^2\varepsilon_{\mathrm{mach}}$ in absolute value, and Lemma~\ref{lem:l2_sum_rounding} adds at most $C_{\mathrm{add}}H^2A^2\varepsilon_{\mathrm{mach}}$.
Since $C_{\mathrm{sum}}=C_{\mathrm{block}}+C_{\mathrm{add}}$, this gives part~(i).

Every visible offset lies between $0$ and $L+1\le N+1<\Pi$. For a wrong position~$j\neq j^\star$, CRT uniqueness therefore gives a prime index~$h^\star$ with
\[
  \delta(j)\not\equiv k+1 \pmod{q_{h^\star}}.
\]
For this mismatched block, Lemma~\ref{lem:perhead} gives
\[
  \langle\mathbf q_t^{(2,h^\star)},\mathbf k_j^{(2,h^\star)}\rangle_p \le -A^2(8/q_{\max}^2-C_{\mathrm{block}}\varepsilon_{\mathrm{mach}}).
\]
By Lemma~\ref{lem:perhead}, every other block, matched or not, contributes at most $C_{\mathrm{block}}A^2\varepsilon_{\mathrm{mach}}$.
The exact sum of the block contributions is therefore at most $-A^2(8/q_{\max}^2-HC_{\mathrm{block}}\varepsilon_{\mathrm{mach}})$, and adding the summation error from Lemma~\ref{lem:l2_sum_rounding} gives
\[
  e^{(2)}_{t,j} \le -A^2\bigl(8/q_{\max}^2-\varepsilon_{\mathrm{score}}\bigr)
  =
  -A^2\Delta.
\]
\end{proof}

Once the wrong logits are at least $U_p$ below the maximum, their exponentials underflow to zero.
The softmax is therefore exactly one-hot on the matching offset.

\begin{lemma}[Layer-2 exact one-hot]\label{lem:l2_exact}
Let copy position~$t$ have offset~$k$, and let $j^\star$ be the unique visible position with $\delta(j^\star)=k+1$.
Provided $k+1\le N+1<\Pi$, Layer~2 satisfies $\alpha^{(2)}_{t,j^\star}=1$ bit-for-bit.
\end{lemma}

\begin{proof}
Lemma~\ref{lem:logit_bounds} shows that $j^\star$ achieves the maximum logit.
After max-shift, its shifted logit is~$0$ and its weight is $\fl(\exp(0))=1$.
For every wrong position,
\[
  \begin{aligned}
    e^{(2)}_{t,j}-e^{(2)}_{t,j^\star}
    &\le -A^2(\Delta-\varepsilon_{\mathrm{score}})\\
    &\le -(2/3)A^2\Delta \le -U_p,
  \end{aligned}
\]
using the standing Layer-2 condition $\varepsilon_{\mathrm{score}}\le\Delta/3$ and the power-of-two choice $A^2\ge4U_p/\Delta$.
Thus every wrong exponential underflows to~$0$, the softmax denominator is~$1$, and $\alpha^{(2)}_{t,j^\star}=1$ bit-for-bit.
\end{proof}

\subsubsection{Copying and resource bounds}

\begin{proof}[Proof of Theorem~\ref{thm:construction}]
We first check the outputs, then account for precision and dimension.
Fix a sufficiently large design horizon~$N$, a finite vocabulary~$\mathcal V$, arbitrary $L\le N$, and arbitrary source tokens $s_1,\ldots,s_L\in\mathcal V$.
At a copy-region query~$t$ of offset $0\le k\le L$, we have $k+1\le N+1<\Pi$, so Lemma~\ref{lem:l2_exact} gives bit-for-bit one-hot attention to the unique visible position~$j^\star$ of offset~$k+1$.
This position holds $s_{k+1}$ if $k<L$ and \(\langle\EoI\rangle\) if $k=L$.
Since $W_O^{(1)}$ writes only to the fingerprint subspace and $\mathrm{MLP}^{(1)}=0$, the Layer-1 skip connection preserves the token embedding at position~$j^\star$.
The token-readout block remains zero through Layer~1.
Layer~2 writes its returned value into that block; its zero MLP and the residual connection therefore leave exactly the retrieved token vector there.
The fixed linear readout from \textup{(B6)} and greedy argmax output the correct vocabulary token, or output \(\langle\EoS\rangle\) when the returned token is \(\langle\EoI\rangle\).

It remains to account for resources.
The first $H$ primes satisfy $q_{\max}=\Theta(H\log H)$ and $\prod_{h\le H}q_h>\!N+1$ for $H=\Theta(\log N/\log\log N)$, by the prime number theorem.
Hence $q_{\max}=\Theta(\log N)$.
With $p_m=c_m\lceil\log_2\log_2\Pi\rceil$ and $c_m\ge4$,
\[
  H^2\varepsilon_{\mathrm{mach}}
  =
  O\!\left(\frac{(\log N)^2}{(\log\log N)^2}\right)(\log N)^{-c_m}
  =
  o(1/\log^2 N),
\]
so $\varepsilon_{\mathrm{score}}=o(1/q_{\max}^2)$, $\Delta>0$, and the cosine gap survives rounding.
The exponent in $A=2^{\lceil\frac12\log_2(4U_p/\Delta)\rceil}$ is $O(\log\log N)$ because $U_p=\Theta(2^{p_e})$ and $\Delta^{-1}=(\log N)^{O(1)}$.
Hence a sufficiently large constant~$c_e$ in $p_e=c_e\lceil\log_2\log_2\Pi\rceil$ gives enough exponent range to represent $A$ and $A^2$ exactly.
Thus all calculations above use $p=O(\log\log N)$ bits.
The added fingerprint subspace has dimension $2H=\Theta(\log N/\log\log N)$.
One more coordinate, fixed to~$1$, supplies the constant entries of the value and query/key projections. The total residual dimension can thus be taken as $d=d_{\mathrm{tok}}(|\mathcal V|)+2H+1$.
\end{proof}

\subsubsection{Implementation and resource conventions}\label{app:vo_implementation}

\paragraph{Modular phase tables and counters.}
Each head~$h$ indexes its rotation table with a counter modulo~$q_h$ that starts at~$0$ and increments by one at each position.
Since $q_h=O(\log N)$, each counter uses $O(\log\log N)$ bits, within the construction's precision scale; total counter state is $O(\log N)$ bits across heads.
The counter supplies the exact residue $t\bmod q_h$ for phase selection, as required by Definition~\ref{def:vo}.
The phase values and rotation arithmetic are rounded to $p$ bits: Lemma~\ref{lem:fingerprint} bounds the total Layer-1 error by $8\varepsilon_{\mathrm{mach}}$, and the Layer-2 gap analysis in~\eqref{eq:eff_gap} carries it through to the logits.
Head~$h$ needs $q_h$ sine/cosine pairs, giving $O(\log^2 N/\log\log N)$ stored scalars across heads, each rounded to $p$ bits.
The counters index these tables and depend only on position, while the attention computation selects the active anchor and forms its relative-offset fingerprint.
Separating phase generation from activation arithmetic also occurs in practical RoPE implementations: Hugging Face Llama computes phases and trigonometric values in float32 before casting them to the activation dtype.\footnote{\href{https://github.com/huggingface/transformers/blob/v4.44.2/src/transformers/models/llama/modeling_llama.py\#L202-L218}{Transformers v4.44.2, \texttt{LlamaRotaryEmbedding.forward}.}}

In exact arithmetic, both RoPE and VO use the composition identity $R^tR^{-j}=R^{t-j}$, but in different computations: RoPE uses rotations in query/key inner products, whereas this construction returns offset representations for the next layer to match.
The construction needs distinct residues to stay distinguishable after rounding.
By Fact~\ref{fact:cos}, distinct residues modulo~$q$ are separated by at least $8/q^2$, so periods of order $\log N$ give a gap of order $1/\log^2N$, which $O(\log\log N)$-bit rounding preserves.

\paragraph{Remark on head-wise periods.}
If a head may use different periods in different two-dimensional blocks, one Layer-1 head suffices.
Every prime head in Definition~\ref{def:vo_prime_full} attends one-hot to the same active anchor, so a single head with the block-diagonal rotation $\operatorname{diag}(R_{q_1},\ldots,R_{q_H})$ reproduces all $H$ rotary outputs, and Layer~2 already combines all prime blocks inside one head via~\eqref{eq:summed_logit}.
The Layer-1 cost is then a rotary dimension of $2H$ rather than $H$ heads.

% NeurIPS-specific checklist retained in the repository, not in the ICLR PDF.
% \newpage
% \input{checklist.tex}

\end{document}